\pdfoutput=1

\documentclass{article}  % sysml

\usepackage{booktabs}
\usepackage{hyperref}

\usepackage[accepted]{icml2021}

\newcommand{\oursp}{\textsc{Maddness}\text{ }}
\newcommand{\ours}{\textsc{Maddness}}
\newcommand{\oursHash}{\textsc{MaddnessHash}}

\usepackage{bbm}  % who knows?

\usepackage{amsmath}          % basic math
\usepackage{amssymb} 			    % math symbols
\usepackage{textcomp}

\usepackage{url}              % better urls; magically keeping doc from breaking due to urls in refs

\usepackage{array}            % multiline table cells; somehow
\usepackage{tabularx}         % better tables
\usepackage{colortbl}
\newcolumntype{Y}{>{\centering\arraybackslash}X}	% centered column type for tabularx

\usepackage{mathtools}          %loads amsmath as well

\DeclareMathOperator*{\argmin}{argmin} % argmin
\usepackage{pbox}   % for trick to force linebreaks in table cells

\usepackage{setspace}

\usepackage{bm}  % bold math

\usepackage[dvipsnames]{xcolor}

\renewcommand{\algorithmiccomment}[1]{
    {\color{OliveGreen}
        \bgroup\hfill//~#1\egroup
    }
}
\newcommand{\LINECOMMENT}[1]{
    {\color{OliveGreen}
        //~#1
    }
}

\usepackage{graphicx}
\usepackage[space]{grffile}   % allow spaces in file names
\graphicspath{{./}}
\DeclareGraphicsExtensions{.pdf,.jpeg,.jpg,.png}

\usepackage[font={bf}]{caption}
\usepackage{paralist}
\setdefaultleftmargin{10pt}{10pt}{}{}{}{}

\usepackage{outlines}
\usepackage{enumitem}
\newcommand{\ItemSpacing}{0mm}
\newcommand{\ParSpacing}{0mm}
\setenumerate[1]{itemsep={\ItemSpacing},parsep={\ParSpacing},label=\arabic*.}
\setenumerate[2]{itemsep={\ItemSpacing},parsep={\ParSpacing}}

\renewcommand{\vec}[1]{\bm{#1}}
\newcommand{\mat}[1]{\bm{#1}}

\newcommand{\cs}[0]{\text{, }}

\providecommand{\ceil}[1]{\left \lceil #1 \right \rceil }
\providecommand{\floor}[1]{\left \lfloor #1 \right \rfloor }

\newcommand\eps\varepsilon

\renewcommand\inf\infty

\newcommand{\w}{\vec{w}}
\renewcommand{\v}{\vec{v}}

\newcommand{\x}{\vec{x}}

\newcommand{\yhat}{\hat{\vec{y}}}

\newcommand{\X}{\mat{X}}

\newcommand{\Sig}{\Sigma}
\newcommand{\Sigm}{\mat{\Sigma}}

\newcommand{\E}{\mathbb{E}}

\newcommand{\A}{\mat{A}}

\renewcommand{\a}{\vec{a}}
\newcommand{\ahat}{\vec{\hat{a}}}
\newcommand{\B}{\mat{B}}

\renewcommand{\b}{\vec{b}}
\newcommand{\g}{\vec{g}}

\newcommand{\U}{\mat{U}}
\newcommand{\Ut}{\mat{U}^{\top}}
\newcommand{\V}{\mat{V}}
\newcommand{\Vt}{\mat{V}^{\top}}
\newcommand{\W}{\mat{W}}
\newcommand{\Y}{\mat{Y}}
\newcommand{\Z}{\mat{Z}}
\newcommand{\I}{\mat{I}}
\renewcommand{\P}{\mat{P}}

\newcommand{\R}{\mathbb{R}}
\DeclarePairedDelimiter\norm{\lVert}{\rVert}%

\newcommand{\Lcal}{\mathcal{L}}

\newcommand{\Scal}{\mathcal{S}}
\newcommand{\Dcal}{\mathcal{D}}
\newcommand{\Fcal}{\mathcal{F}}

\newcommand{\iid}{\stackrel{iid}{\sim}}
\usepackage{amsthm}           % theorems
\newtheorem{theorem}{Theorem}[section]
\newtheorem{lemma}{Lemma}[section]
\newtheorem{definition}{Definition}[section]

\newtheorem*{theorem*}{Theorem}  % no numbering

\icmltitlerunning{Multiplying Matrices Without Multiplying}

\begin{document}

\twocolumn[
% ================================================================
\icmltitle{Multiplying Matrices Without Multiplying}
% ================================================================

\begin{icmlauthorlist}
\icmlauthor{Davis Blalock}{mosaic,csail}
\icmlauthor{John Guttag}{csail}
\end{icmlauthorlist}

\icmlaffiliation{mosaic}{MosaicML, San Francisco, CA, USA}
\icmlaffiliation{csail}{MIT CSAIL, Cambridge, MA, USA}
\icmlcorrespondingauthor{Davis Blalock}{davis@mosaicml.com}

\icmlkeywords{Vector Quantization, Approximate Algorithms, Matrix Multiplication}

% \end{icmlauthorlist}

% \icmlaffiliation{WayneEnterprises}{Wayne Enterprises, Gotham, USA}
% \icmlcorrespondingauthor{Batman}{batman@batman.batman}

\icmlkeywords{Vector Quantization, Approximate Algorithms, Matrix Multiplication}

\vskip 0.3in
] % end of icml 2 columnn

\printAffiliationsAndNotice{}

% ------------------------------------------------
\begin{abstract}
% ------------------------------------------------

% We introduce an approximate matrix multiplication algorithm that greatly outperforms existing methods. Experiments using hundreds of matrices from diverse domains show that it often runs $10\times$ faster than current methods at a given level of error, as well as $100\times$ faster than exact matrix multiplication. In the common case that one matrix is known ahead of time, our method also has the interesting property that it requires zero multiply-adds.

Multiplying matrices is among the most fundamental and compute-intensive operations in machine learning. % Fortunately, for many applications, approximate matrix multiplication suffices.
% Consequently, many algorithms use approximate
% and scientific computing.
% Consequently, the task of efficiently approximating matrix products has received significant attention.
% This has led researchers to consider using approximate matrix products.
% Consequently, there has been a great deal of research on approximate matrix multiplication (AMM).
Consequently, there has been significant work on efficiently approximating matrix multiplies.
We introduce a learning-based algorithm for this task that greatly outperforms existing methods.
% We introduce an approximate matrix multiply algorithm that greatly outperforms existing methods.
Experiments using hundreds of matrices from diverse domains show that it often runs $100\times$ faster than exact matrix products and $10\times$ faster than current approximate methods. In the common case that one matrix is known ahead of time,
our method also has the interesting property that it requires zero multiply-adds.
These results suggest that a mixture of hashing, averaging, and byte shuffling—--the core operations of our method—--could be a more promising building block for machine learning than the sparsified, factorized, and/or scalar quantized matrix products that have recently been the focus of substantial research and hardware investment.
% dozens of papers and billions of dollars of hardware investment in recent years.
%They also suggest that machine learning can be applied at even the lowest level of the software stack---e.g., our matrix product kernels begin by running inference in decision trees.

% These results suggest that a mixture of hashing, averaging, and byte shuffling—--the core operations of our method—--could be a more promising building block than the sparsified, factorized, and/or scalar quantized matrix products that have been the subject of hundreds of papers and enormous hardware investment in recent years.

% Multiplying matrices is among the most fundamental and most computationally demanding operations in machine learning and scientific computing.

% We introduce a new approximate matrix multiplication algorithm that greatly outperforms existing methods. Experiments using hundreds of matrices from diverse domains show that it often runs 100× faster than exact matrix multiplication and 10× faster than current approximate methods. In the common case that one matrix is known ahead of time, which occurs when applying a linear model, our method also has the interesting property that it requires zero multiply-adds.

% Our results suggest that a mixture of hashing, averaging, and byte shuffling, the core operations of our method, could be a more promising building block than the sparsified, factorized, and/or scalar quantized matrix products that have been commonly proposed in the literature and have driven much hardware investment in recent years.

\end{abstract}

% ================================================================
\vspace*{-5mm}
\section{Introduction} \label{sec:intro}
\vspace{-.5mm}
% ================================================================

Matrix multiplication is among the most fundamental subroutines used in machine learning and scientific computing. As a result, there has been a great deal of work on implementing high-speed matrix multiplication libraries \cite{pytorch,eigen,tensorflow}, designing custom hardware to accelerate multiplication of certain classes of matrices \cite{eie,eyeriss,scnn,tpu},
speeding up distributed matrix multiplication \cite{distributedCoded, shortDot, entangledPolynomial, matmulCommunicationBounds},
and designing efficient Approximate Matrix Multiplication (AMM) algorithms under various assumptions% and problem settings.
.

We focus on the AMM task under the assumptions that the matrices are tall, relatively dense, and resident in a single machine's memory. In this setting, the primary challenge is minimizing the amount of compute time required to approximate linear operations with a given level of fidelity.
%not reduction of disk or network usage [], efficient coordination between distributed workers [], maximization of a space-distortion tradeoff [], or reduction of asymptotic complexity []. Instead, it is minimization of the amount of CPU time required to approximate linear operations with a given level of fidelity.
This setting arises naturally in machine learning and data mining when one has a data matrix $\mat{A}$ whose rows are samples and a linear operator $\mat{B}$ one wishes to apply to these samples. $\mat{B}$ could be a linear classifier, linear regressor, or an embedding matrix, among other possibilities.

As a concrete example, consider the task of approximating a softmax classifier trained to predict image labels given embeddings derived from a neural network. Here, the rows of $\A$ are the embeddings for each image, and the columns of $\B$ are the weight vectors for each class. Classification is performed by computing the product $\A\B$ and taking the argmax within each row of the result.
In Figure~\ref{fig:fig1}, we see the results of approximating $\A\B$ using our method and its best-performing rivals \cite{hashjl, sparsePCA} on the CIFAR-10 and CIFAR-100 datasets.
\vspace{1mm}
\begin{figure}[h]
\begin{center}
\includegraphics[width=.95\linewidth]{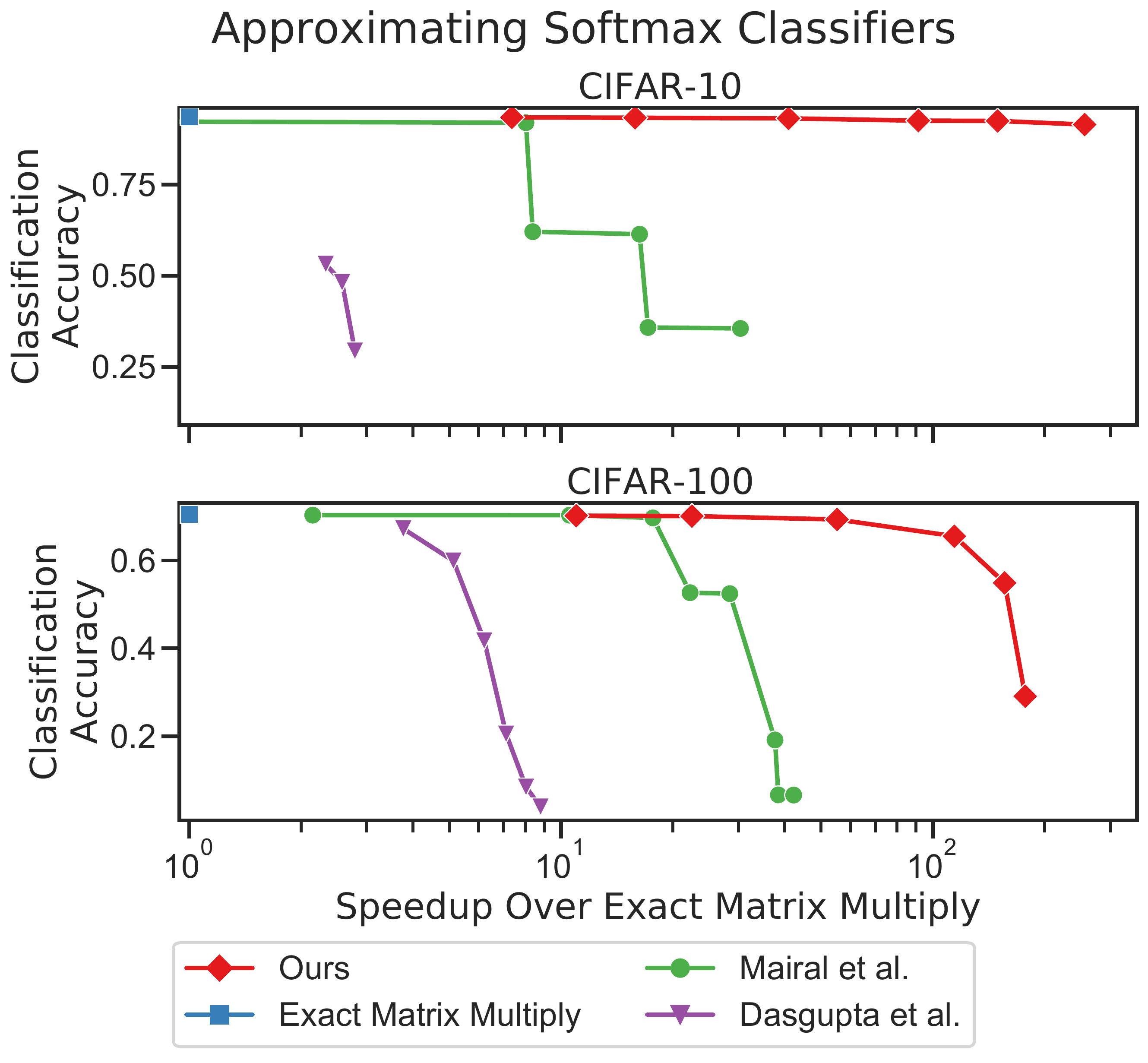}
\caption{Our method achieves a dramatically better speed-accuracy tradeoff than the best existing methods when approximating two linear classifiers.}
\label{fig:fig1}
\end{center}
\end{figure}
\vspace{-1mm}
% Each method yields a curve whose points are specific speedups and associated levels of classification accuracy. More speedup and more accuracy (up and to the right) is better. While we defer detailed discussion to Section~\ref{sec:results}, it is clear that our proposed approach significantly outperforms alternatives.

% In addition to achieving strong empirical performance, our method also has an

% In addition to achieving strong empirical performance,
Our method represents a significant methodological departure from most traditional approaches to this problem. Traditional AMM methods construct matrices $\V_A, \V_B \in \R^{D \times d}, d \ll D$ such that
\begin{align}
    \A \B \approx (\A \V_A) (\V_B^\top \B).
\end{align}
Often, $\V_A$ and $\V_B$ are sparse, embody some sort of sampling scheme, or have other structure such that these projection operations are faster than a dense matrix multiply. In short, these methods use linear functions to preprocess $\A$ and $\B$ and reduce the problem to exact matrix multiplication in a lower-dimensional space.
% reduce the AMM problem to exact matrix multiplication in a lower-dimensional space using linear functions to preprocess $\A$ and $\B$.

Our proposed method, \ours\footnote{Multiply-ADDitioN-lESS}, instead employs a \textit{nonlinear} preprocessing function and reduces the problem to table lookups. Moreover, in the case that $\B$ is known ahead of time---which happens when applying a trained linear model to new data, among other situations---\oursp does not require any multiply-add operations.
Our method is most closely related to vector quantization methods used for similarity search (e.g., \cite{bolt,quickAdc,quickerAdc,pq,opq}). However, instead of using an expensive quantization function that requires many multiply-adds, we introduce a family of quantization functions that require no multiply-adds. %, including binary \texttt{XNOR} and \texttt{popcount} operations (though excluding implementation-level multiplies to compute memory addresses).
% TODO something about multiplication taking a lot of transistors vs table lookups.

Our contributions can be summarized as follows:
\vspace{-3mm}
\begin{itemize}\itemsep-1mm
    \item An efficient family of learned vector quantization functions that can encode over 100GB of data per second in a single CPU thread.
    % \item A provably good procedure for learning one of these functions from training data, in addition to other theoretical analysis.
    \item A high-speed summation algorithm for low-bitwidth integers that avoids upcasting, saturation, and overflow.
    \item An algorithm based on these functions for approximate matrix multiplication. Experiments across hundreds of diverse matrices demonstrate that this algorithm significantly outperforms existing alternatives. It also features theoretical quality guarantees.
\end{itemize}
\vspace{-3mm}

% ------------------------------------------------
\subsection{Problem Formulation} \label{sec:problemStatement}
% ------------------------------------------------

% Let $\A \in \R^{N \times D}$ and $\B \in \R^{D \times M}$ be two matrices, with $N \gg D, M$, and $M$ not significantly larger than $D$. Given a computation time budget $\tau$, our task is to
Let $\A \in \R^{N \times D}$ and $\B \in \R^{D \times M}$ be two matrices, with $N \gg D \ge M$. Given a computation time budget $\tau$, our task is to
construct three functions $g(\cdot)$, $h(\cdot)$, and $f(\cdot)$, along with constants $\alpha$ and $\mat{\beta}$, such that
\begin{align} \label{eq:objective}
    \norm{\alpha f(g(\A), h(\B)) + \mat{\beta} - \A\B}_F < \eps(\tau) \norm{\A\B}_F
\end{align}
for as small an error $\eps(\tau)$ possible. The constants $\alpha$ and $\mat{\beta}$ are separated from $f(\cdot,\cdot)$ so that $f(\cdot,\cdot)$ can produce low-bitwidth outputs (e.g., in the range $[0, 255]$) even when the entries of $\A\B$ do not fall in this range.

We assume the existence of a training set $\tilde{\A}$, whose rows are drawn from the same distribution as the rows of $\A$. This is a natural assumption in the case that rows of $\A$ represent examples in training data, or structured subsets thereof (such as patches of images).

% This assumption is common in the information retrieval literature \cite{bolt,pairq,quip}, though not in most theoretical work. Lastly, while our ideas are not specific to CPUs, we focus on CPU performance throughout this paper and leave implementation on GPUs, FPGAs, etc., to future work.

% While we believe that our results suggest that methods similar to our own could hold promise for accelerating convolution, deep learning, and other workloads bottlenecked by linear transforms, such extensions will require additional research---the present work only claims superiority for \textit{approximate matrix multiplication} as described in our problem formulation.

% ================================================================
% \section{Background and Related Work}
\vspace{-2mm}
\section{Related Work} \label{sec:relatedWork}
\vspace{-.5mm}
% ================================================================

Because our work draws on ideas from randomized algorithms, approximate matrix multiplication, vector quantization, and other fields, the body of work related to our own is vast. Here, we provide only a high-level overview, and refer the interested reader to \cite{learningToHashSurvey, hashingSimilaritySurvey, isvd} for more detailed surveys. We also defer discussion of related vector quantization methods to the following sections.
%, since it is easier to appreciate how they differ from our own method once our method has been introduced.

\vspace{-1mm}
\subsection{Linear Approximation}
% \vspace{-.5mm}
Most AMM methods work by projecting $\A$ and $\B$ into lower-dimensional spaces and then performing an exact matrix multiply.
%  $\V_A, \V_B \in \R^{D \times d}, d \ll D$ such that
% % \begin{align}
% %     \A \B \approx (\A \V_A) (\V_B^\top \B).
% % \end{align}
% There are many options for choosing $\V_A$ and $\V_B$
%, depending on what assumptions one makes. The most common case studied is that in which $N$, $D$, and $M$ are large, but one has little or no prior knowledge about either matrix.
% One simple approach is to choose them such that $(\A \V_A)$ and $(\V_B^\top \B)$ preserve certain structure in $\A$ and $\B$---i.e. to sketch each matrix independently.
One simple option for choosing the projection matrices is to use matrix sketching algorithms. The most prominent deterministic matrix sketching methods are the Frequent Directions algorithm \cite{liberty_simple_2012, ghashami_frequent_2016} and its many variations \cite{teng_fast_2019, francis_practical_2018, ye_frequent_2016, huang_near_2019, luo_robust_2019, francis_improvement_2018}. There are also many randomized sketching methods \cite{sarlos_improved_2006, kyrillidis_approximate_2014, pagh_compressed_2013, hashjl,osnap} and sampling methods \cite{drineas_fast_2006-1, drineas_fast_2006-2}.

A weakness of matrix sketching methods in the context of matrix multiplication is that they consider each matrix in isolation. To exploit information about both matrices simultaneously, \citet{drineas_fast_2006} sample columns of $\A$ and rows of $\B$ according to a sampling distribution dependent upon both matrices. Later work by \citet{manne_fast_2014} reduces approximation of the matrices to an optimization problem, which is solved by steepest descent. \citet{mroueh_co-occuring_2016}, \citet{ye_frequent_2016}, and \citet{francis_improvement_2018} introduce variations of the Frequent Directions algorithm that take into account both matrices.

All of the above methods differ from our own not merely in specifics, but also in problem formulation. These methods all assume that there is no training set $\tilde{\A}$ and nearly all focus on large matrices, where provably reduced asymptotic complexity for a given level of error is the goal. % We assume smaller matrices and the presence of a training set.

\vspace{-1mm}
\subsection{Hashing to Avoid Linear Operations}
% \vspace{-.5mm}
In the neural network acceleration literature, there have been several efforts to accelerate dense linear layers using some form of hashing \cite{springScalable,slide,wtaSoftmax,googleWtaCvpr,hashnet}. These methods differ from our own in the hash functions chosen, in not exploiting a training set, and in the overall goal of the algorithm. While we seek to approximate the entire output matrix, these methods seek to either sample outputs \cite{springScalable,slide}, approximate only the largest outputs \cite{wtaSoftmax,googleWtaCvpr}, or implement a fixed, sparse linear operator \cite{hashnet}.

% \input{bg.tex}

%================================================================
\vspace{-1.5mm}
\section{Background - Product Quantization} \label{sec:background}
\vspace{-.5mm}
%================================================================

% Recall that our task is to construct functions $g(\cdot)$, $h(\cdot)$, and $f(\cdot)$ such that
% \begin{align}
%     \norm{f(g(\A), h(\B)) - \A\B}_F < \eps(\tau) \norm{\A\B}_F
% \end{align}
% for the smallest $\eps(\tau)$ possible.

% In this section, we discuss the vector quantization (VQ) approach to approximate matrix multiplication. Because these methods approximate matrix products by approximating individual dot products between rows of $\A$ and columns of $\B$, we do so by examining how they approximate a single dot product $\a^\top \b$ between one row of $\A$ and one column of $\B$. For simplicity, we will also begin by focusing on Product Quantization (PQ) \cite{pq}, the classic algorithm on which most others are based.

To lay the groundwork for our own method, we begin by reviewing Product Quantization (PQ) \cite{pq}. PQ is a classic vector quantization algorithm for approximating inner products and Euclidean distances and serves as the basis for nearly all vector quantization methods similar to our own. % This not only lays necessary groundwork for explaining our own method, but also affords the opportunity to discuss

% to both lay the groundwork  the classic vector quantization algorithm for approximating inner products on which most others are based.

The basic intuition behind PQ is that $\a^\top \b \approx \hat{\a}^\top \b$, where $\norm{\hat{\a} - \a}$ is small but $\hat{\a}$ has special structure allowing the product to be computed quickly. This structure consists of $\hat{\a}$ being formed by concatenating learned prototypes in disjoint subspaces; one obtains a speedup by precomputing the dot products between $\b$ and the prototypes once, and then reusing these values across many $\a$ vectors. The $\a$ vectors here are the (transposed) rows of $\A$ and the $\b$ vectors are the columns of $\B$.

 % \cite{pq}, a classic approach to this problem that serves as a conceptual foundation for our own method. There are three basic steps behind PQ:
In somewhat more detail, PQ consists of the following:
\vspace{-3mm}
% \begin{enumerate}\itemsep.5mm
\begin{enumerate}\itemsep1.5mm
    % \item asdf
    % \item $g(\A)$
    % \item Replacing each row of $\A$ with a prototype chosen from a predefined set.%
    \item \textbf{Prototype Learning} - In an initial, offline training phase,  cluster the rows of $\A$ (or a training set $\tilde{\A}$) using K-means to create prototypes. A separate K-means is run in each of $C$ disjoint subspaces to produce $C$ sets of $K$ prototypes. %The prototypes consist of the cartesian product of prototypes within disjoint subspaces.
    % \item $g(\A)$ - Replace each row of $\A$ with the most similar prototype.
    % \item $h(\B)$ - Precompute the dot products between each column of $\B$ and each prototype.
    \item \textbf{Encoding Function}, $g(\a)$ - Determine the most similar prototype to $\a$ in each subspace. Store these assignments as integer indices using $C \log_2(K)$ bits.
    \item \textbf{Table Construction}, $h(\B)$ - Precompute the dot products between $\b$ and each prototype in each subspace. Store these partial dot products in $C$ lookup tables of size $K$.
    \item \textbf{Aggregation}, $f(\cdot,\cdot)$ - Use the indices and tables to \textit{lookup} the estimated partial $\a^\top \b$ in each subspace, then sum the results across all $C$ subspaces. %$A[i]^\top B[:, j]$\footnote{Because we will frequently need to refer to slices of rank-3 tensors and above, we employ Python-style slicing notation rather than traditional superscripts and subscripts.} rather than compute it with a series of multiply-adds.
\end{enumerate}
\vspace{-3mm}
PQ is depicted for a single pair of vectors $\a$ and $\b$ in Figure~\ref{fig:pq}. We elaborate upon each of these steps below.
% \vspace{-2mm}

% ------------------------------------------------
% \subsection{Prototype Learning}
\vspace{-2mm}
\paragraph{Prototype Learning:}
% ------------------------------------------------

Let $\tilde{A} \in \R^{N \times D}$ be a training set, $K$ be a number of prototypes per subspace, $C$ be a number of subspaces, and $\{\mathcal{J}^{(c)}\}_{c=1}^C$ be the mutually exclusive and collectively exhaustive sets of indices associated with each subspace. The training-time task of PQ is to learn $C$ sets of prototypes $\mat{P}^{(c)} \in \R^{K \times |\mathcal{J}^{(c)}|}$ and assignments $\vec{z}^{(c)} \in \R^{N}$ such that:
\vspace{-2mm}
\begin{align}
    \sum_{i=1}^N \sum_{c=1}^C \sum_{j\in \mathcal{J}^{(c)}} \left( \tilde{\A}_{ij} - \mat{P}^{(c)}_{z^{(c)},j} \right)^2
\vspace*{-.5mm}
\end{align}
is minimized. It does this by running K-means separately in each subspace $\mathcal{J}^{(c)}$ and using the resulting centroids and assignments to populate $\mat{P}^{(c)}$ and $\vec{z}^{(c)}$.

\begin{figure}[t]
\begin{center}
\includegraphics[width=\linewidth]{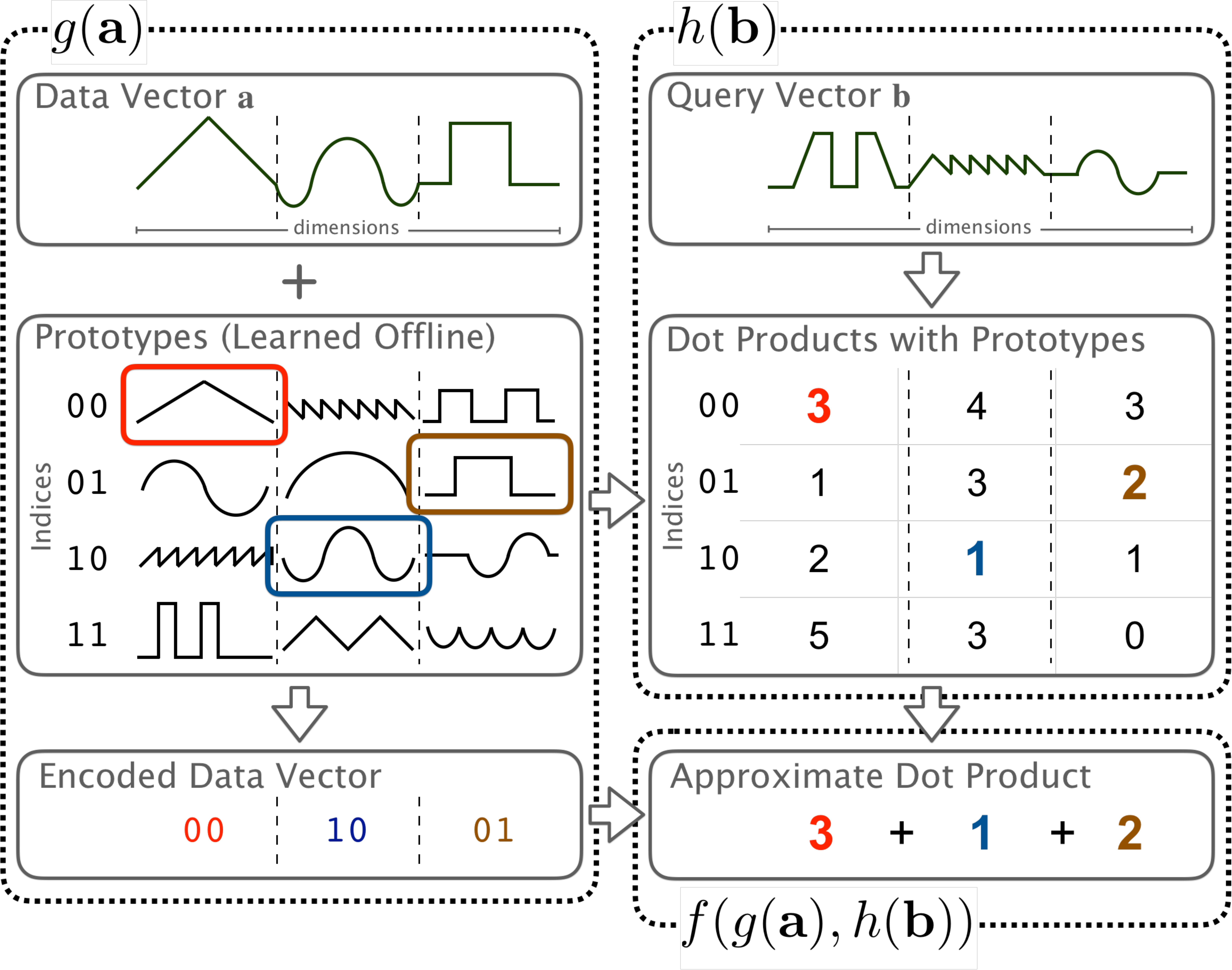}
\caption{Product Quantization. The $g(\cdot)$ function returns the index of the most similar prototype to the data vector a in each subspace. The $h(\cdot)$ function computes a lookup table of dot products between the query vector b and each protoype in each subspace. The aggregation function $f(\cdot,\cdot)$ sums the table entries corresponding to each index.}
% $\vec{a}$
\label{fig:pq}
\vspace*{2mm}
\end{center}
\end{figure}

% Perhaps the simplest means of constructing prototypes would be to run a clustering algorithm like K-means on the rows of $\A$. This would be effective, but would require using a large number of prototypes in order to have each row closely match its prototype. What would be preferable is some means of having a huge number of prototypes without having to pay so large a time and space cost.

% PQ achieves this goal by using as prototypes the cartesian product of prototypes within disjoint subspaces. Concretely, PQ runs K-means with $K$ centroids in each of $C$ disjoint (usually contiguous) sets of dimensions. This results in $K^C$ possible combinations of centroid assignments for each vector, with each unique combination corresponding to a unique overall prototype.

% vectors into disjoint ``subvectors'' (each corresponding to a unique set of dimensions) and runs k-means within each subspace.
%  If there are $K$ prototypes per subspace and $C$ subspaces, this results in $K^C$ possible combinations of prototype assignments for each vector. Viewed differently, this creates $K^C$ overall prototypes whose entries

% ------------------------------------------------
% \subsection{Encoding Function - $g(\A)$}
% \vspace{-2.5mm}
% \vspace{144pt}
% \vspace{288pt}
\paragraph{Encoding Function, $\bm{g(\A)}$:}
% ------------------------------------------------

Given the learned prototypes, PQ replaces each row $\a$ of $\A$ with the concatenation of its $C$ K-means centroid assignments in each of the $C$ subspaces. Formally:
\vspace{-.5mm}
\begin{align} \label{eq:pqLoss}
    g^{(c)}(\a) \triangleq \argmin_k \sum_{j\in \mathcal{J}^{(c)}} \left( \a_j - \mat{P}^{(c)}_{k,j} \right)^2.
\vspace*{-.5mm}
\end{align}
We will refer to the resulting sequence of indices as the \textit{encoding} of $\a$ and the set of $K$ centroids as a \textit{codebook}. For convenience, we will also refer to the vector $\a^{(c)} \triangleq \langle a_j \rangle, j \in \mathcal{J}^{(c)}$ as the \textit{subvector} of $\a$ in subspace $c$. %This encoding function has complexity $\Theta(KD)$ for each $\a$, since there are $K$ centroids per subspace and the sum of the number of indices in all subspaces is $D$.

% ------------------------------------------------
% \subsection{Table Construction - $h(\B)$}
% \vspace{-2.5mm}
\paragraph{Table Construction, $\bm{h(\B)}$:}
% ------------------------------------------------

Using these same prototypes, PQ constructs a lookup table $h^{(c)}(\b) \in \R^K$ in each of the $C$ subspaces for each column $\b$ of $\B$, where
% \vspace{-.5mm}
\begin{align}
    h^{(c)}(\b)_k \triangleq \sum_{j\in \mathcal{J}^{(c)}} \b_j \mat{P}^{(c)}_{k,j}.
\end{align}
Existing work has shown that setting $K = 16$ and quantizing the lookup tables to 8 bits can offer enormous speedups compared to larger $K$ and/or floating-point tables \cite{bolt, quickAdc, quickerAdc}. This is because 16 1-byte entries can be stored in a SIMD register, allowing 16 or more table lookups to be performed in parallel using a byte shuffle instruction. Since the table entries naturally occupy more than 8 bits even for 8-bit data, some means of quantizing these entries is necessary. This can easily be done by subtracting off the minimum entry in each table and linearly rescaling such that the maximum entry in any table is at most 255. Ignoring rounding error, this affine transform is invertible, and is reflected by the constants $\alpha$ and $\mat{\beta}$ in equation~\ref{eq:objective}. See Appendix~\ref{sec:lutQuantize} for further details.

% \vspace{-.5mm}
%This has complexity $\Theta(KD)$ for each $\b$, since it is essentially the same as $g(\cdot)$ except without the $\argmin$ operation.

% ------------------------------------------------
% \subsection{Aggregation - $f(\cdot,\cdot)$}
% \vspace{-2mm}
\paragraph{Aggregation, $\bm{f(\cdot,\cdot)}$:}
% ------------------------------------------------

Given the encoding of $\a$ and the lookup tables for $\b$, the product can be approximated as
% \vspace*{-2mm}
\begin{align}
    \a^\top \b = \sum_{c=1}^C \a^{(c)\top} \b^{(c)} \approx \sum_{c=1}^C h^{(c)}(\b)_k, \text{ }k = g^{(c)}(\a).
\end{align}
\vspace{-3mm}
\section{Our Method} \label{sec:method}
\vspace{-.5mm}
%================================================================

% In the previous section, we saw that Product Quantization (PQ) has time complexity $\Theta(NDK + DKC + NMC)$. For the $N, M \gg D$ case considered in existing literature, the preprocessing time contributing the first two terms is negligible and the complexity is effectively $\Theta(NMC)$. Unfortunately, when not running many queries (large $M$) on huge databases (large $N$), this case often does not hold. As discussed previously, we would like an algorithm that requires only $N \gg M, D$. In this setting, the preprocessing time $g(\A)$ can be significant, since $ND$ may be similar to (or even larger than) $NM$.

Product Quantization and its variants yield a large speedup with $N, M \gg D$. However, we require an algorithm that only needs $N \gg M, D$, a more relaxed scenario common when using linear models and transforms. In this setting, the preprocessing time $g(\A)$ can be significant, since $ND$ may be similar to (or even larger than) $NM$.

To address this case, we introduce a new $g(\A)$ function that yields large speedups even on much smaller matrices. %In this section, we describe how \oursp builds upon the foundation of PQ to achieve this. We also discuss aspects of \oursp that offer further enhancements that are applicable to vector quantization approaches in general.
The main idea behind our function is to determine the ``most similar'' prototype through locality-sensitive hashing \cite{lshOrig}; i.e., rather than compute the Euclidean distance between a subvector $\a^{(c)}$ and each prototype, we hash $\a^{(c)}$ to one of $K$ buckets where similar subvectors tend to hash to the same bucket. The prototypes are set to the means of the subvectors hashing to each bucket.%, rather than with K-means.

% Each bucket is associated with a prototype, defined to be the mean of all training vectors hashing to that bucket.

% ------------------------------------------------
% \vspace{-1mm}
\subsection{Hash Function Family, $\bm{g(\cdot)}$}
% ------------------------------------------------

% A key question in this approach is how to choose the locality-sensitive hash function.
% There are two types of these functions to choose from: data-independent and data-dependent []. The former assume no training set and are typically based on randomization []. The latter assume a training set, and most often involve training a neural network to produce binary codes []. We refer the reader to [] for a survey. As one might expect, data-dependent hash functions yield higher quality results, but tend to run more slowly and lack formal guarantees.
Because we seek to exploit a training set while also doing far less work than even a single linear transform, we found that existing hash functions did not meet our requirements. Consequently, we designed our own family of trainable hash functions.
% This function family may be of independent interest, but we leave exploration of its efficacy in other contexts to future work.
%We do not claim that this method outperforms all possible alternatives---merely that it enables excellent results in our experiments.
The family of hash functions we choose is balanced binary regression trees, with each leaf of the tree acting as one hash bucket. The leaf for a vector $\x$ is chosen by traversing the tree from the root and moving to the left child if the value $x_j$ at some index $j$ is below a node-specific threshold $v$, and to the right child otherwise. To enable the use of SIMD instructions, the tree is limited to 16 leaves and all nodes at a given level of the tree are required to split on the same index $j$. The number 16 holds across many processor architectures, and we refer the reader to Appendix~\ref{sec:hashQuantize} for further vectorization details.

% hack to get perfect spacing
% \algnewcommand{\COMMENTT}[2][.5\linewidth]{\leavevmode\hfill\makebox[#1][l]{\hphantom{a} // ~#2}}

% Formally, consider a set of four indices ${j_1,\ldots,j_d}$ and four arrays of split values $\vec{v}_1,\ldots,\vec{v}_d$, with $v_d^\prime$ having length $2^{d^\prime} - 1$. A vector $\x$ is mapped to an index using Algorithm~\ref{algo:ourEnc}.
Formally, consider a set of four indices ${j^1,\ldots,j^4}$ and four arrays of split thresholds $\vec{v}^1,\ldots,\vec{v}^4$, with $v_t$ having length $2^{t-1}$. A vector $\x$ is mapped to an index using Algorithm~\ref{algo:ourEnc}.
% \begin{centering}
\begin{algorithm}[t]
\caption{\oursHash} \label{algo:ourEnc}
\begin{algorithmic}[1]
    \STATE {\bfseries Input:} vector $\x$, split indices ${j^1,\ldots,j^4}$, split thresholds $\vec{v}^1,\ldots,\vec{v}^4$
    % \STATE {$i \leftarrow 0$}  \COMMENT{Node Index Within Level (Zero-Indexed)}
    \STATE {$i \leftarrow 1$}  \COMMENT{node index within level of tree}
    % \STATE {$i \leftarrow 1$}  \bgroup\hfill//~Foo\egroup
    \FOR{$t \leftarrow 1 \textbf{ to } 4$}
    % // for each level of tree
    % \STATE{$v \leftarrow \vec{v}^t_i$}   %\phantom{    }// lookup split value based on current $i$
    % \STATE{$b \leftarrow a_{j^t} \ge v \text{ ? } 1 \text{ : } 0$}  %\phantom{  }// above split value?
    \STATE{$v \leftarrow \vec{v}^t_i$}   \COMMENT{lookup split threshold for node $i$ at level $t$}
    \STATE{$b \leftarrow x_{j^t} \ge v \text{ ? } 1 \text{ : } 0$}  \COMMENT{above split threshold?}
    % \STATE{$i \leftarrow 2i - 1 + b$}  %\phantom{    }// node index in next tree level
    \STATE{$i \leftarrow 2i - 1 + b$} \COMMENT{assign to left or right child}
    \ENDFOR
    \STATE{ \textbf{return} $i$}
    % \RETURN {$i + 1$}  \COMMENT{One-indexed}
    % \RETURN {$i$}
\end{algorithmic}
\end{algorithm}
% \vspace*{5mm}
% \end{centering}
% \vspace{-3mm}
This function is simple, only depends on a constant number of indices in the input vector, and can easily be vectorized provided that the matrix whose rows are being encoded is stored in column-major order. % The only subtlety in doing this vectorization is that one may need to convert floating point data to integers depending on one's processor; see Appendix~\ref{sec:hashQuantize} for further discussion.

% TODO move below to appendix
% (though in practice, $\gamma_j$ and $\delta_j$ are applied using a fused-multiply-add.
% first computing the minimum and maximum split value within each $\vec{v}$, subtracting the smallest value

% ------------------------------------------------
% \vspace{-2mm}
% \vspace{4mm}
% \phantom{\newline}
\subsection{Learning the Hash Function Parameters}
% ------------------------------------------------

% The split indices ${j^1,\ldots,j^4}$ and split values $\vec{v}^1,\ldots,\vec{v}^4$ are optimized on the training set $\tilde{\A}$ using a greedy algorithm. This algorithm begins with a single-node tree and iteratively adds levels by splitting the current tree's leaves. It tracks which rows of $\tilde{\A}$ are assigned to each leaf node in order to assess the quality of possible splits. To generate possible splits, the algorithm proposes a small number of indices using a heuristic and determines the optimal split values for each index.% , and chooses the index and values that yield the lowest loss. This loss is defined
The split indices ${j^1,\ldots,j^4}$ and split thresholds $\vec{v}^1,\ldots,\vec{v}^4$ are optimized on the training matrix $\tilde{\A}$ using a greedy tree construction algorithm.
%  and chooses the best among those considered
% it will be helpful to
To describe this algorithm, we introduce the notion of a \textit{bucket} $\mathcal{B}^t_i$, which is the set of vectors mapped to node $i$ in level $t$ of the tree. The root of the tree is level 0 and $\mathcal{B}^0_1$ contains all the vectors. It will also be helpful to define the sum of squared errors (SSE) loss associated with a bucket, or a specific (index, bucket) pair:
\begin{equation}
\setlength{\jot}{10pt}
\begin{aligned}
% \text{\vspace{-3mm}}
    \mathcal{L}(j \text{, } \mathcal{B}) &\triangleq \sum_{\x \in \mathcal{B}} \left( x_j - \frac{1}{|\mathcal{B}|}\sum_{\x^\prime \in \mathcal{B}} x^\prime_j \right)^2  \\
    \mathcal{L}(\mathcal{B}) &\triangleq \sum_j \mathcal{L}(j \text{, } \mathcal{B}).
% \vspace*{-3mm}
\end{aligned}
\end{equation}
% I.e., this loss for each index is the sum of squared errors (SSE) when estimating each $x_j$ as equal to its mean value within the bucket, and the overall loss for a bucket is the sum of the values in each dimension. The overall loss can also be understood as the ``energy'' of the bucket [].
Using this notation, it suffices to characterize the learning algorithm by describing the construction of level $t$ of the tree given the buckets $\mathcal{B}^{t-1}_1,\ldots,\mathcal{B}^{t-1}_{2^{t-1}}$ from the previous level. This procedure is given in Algorithm~\ref{algo:learnTree}.

% In line~\ref{line:dimHeuristic},
% The construction of level $t$ of the tree is given in Algorithm~\ref{algo:learnTree}.
In line~2, we select a fixed number of indices to evaluate. Several heuristics are possible, including evaluating all indices. We found that simply selecting the top $n$ indices that contributed the most loss summed across all buckets was difficult to beat. In preliminary experiments, we found that using $n > 4$ indices offered no clear benefit (and even choosing $n = 1$ was nearly as good), so we fix $n = 4$.
% The loss associated with a given index in a given bucket is the variance within that dimension multiplied by the size of the bucket---i.e., the sum of squared errors compared to the bucket's mean.

% In lines~\ref{line:dimEvalStart}-\ref{line:dimEvalEnd},
In lines~4-15, we find the minimal loss obtainable by splitting all buckets along that index, but with bucket-specific cutoffs. This loss is minimal not in the sense that it leads to the globally optimal tree, but in that it minimizes the sum of the losses in the buckets produced in this iteration. To do this, we invoke the subroutine \texttt{optimal\_split\_threshold}, which takes in a bucket $\mathcal{B}$ and an index $j$ and tests all possible thresholds to find one minimizing $\mathcal{L}(j \text{, } \mathcal{B})$. This can be done in time $O(|\mathcal{J}^{(c)}||\mathcal{B}| \log(|\mathcal{B}|))$. The pseudocode for this subroutine is given in Algorithms~\ref{algo:optimalSplitVal}~and~\ref{algo:cumSSE} in Appendix~\ref{sec:optimalSplitVal}.

% (lines~\ref{line:splitBucketsStart}-\ref{line:splitBucketsEnd})
Once a split index $j$ and an array of split thresholds $\vec{v}$ are chosen, all that remains is to split the buckets to form the next level of the tree (lines~16-21). This entails forming two child buckets from each current bucket by grouping vectors whose $j$th entries are above or below the bucket's split threshold.% , where the first child has vectors whose elements $x_j$ are less than the bucket's split value and the second child has those whose elements are above the bucket's split value.

% Once the full tree has been constructed, the prototypes are set to the means of each bucket.

% \vspace{1mm}
\begin{algorithm}[t]
\caption{Adding The Next Level to the Hashing Tree} \label{algo:learnTree}
\begin{algorithmic}[1]
    \STATE {\bfseries Input:} buckets $\mathcal{B}^{t-1}_1,\ldots,\mathcal{B}^{t-1}_{2^{t-1}}$, training matrix $\tilde{\A}$

    \LINECOMMENT{greedily choose next split index and thresholds}
    \STATE{$\mathcal{\hat{J}} \leftarrow \texttt{heuristic\_select\_idxs}(
        \mathcal{B}^{t-1}_1,\ldots,\mathcal{B}^{t-1}_{2^{t-1}})$}\label{line:dimHeuristic}
    \STATE{$l^{min} \text{, } j^{min}, \vec{v}^{min} \leftarrow \inf \text{, NaN, NaN}$}
    \FOR{$j \in \mathcal{\hat{J}}$} \label{line:dimEvalStart}
        \STATE{$l \leftarrow 0 $}  \COMMENT{initialize loss for this index to 0}
        \STATE{$\vec{v} \leftarrow \text{[ ]}$} \COMMENT{empty list of split thresholds}
        \FOR{$i \leftarrow 1 \textbf{ to } 2^{t-1} $}
            \STATE{$v_i \text{, } l_i \leftarrow \texttt{optimal\_split\_threshold}(j \text{, } \mathcal{B}^{t-1}_i) $}
            \STATE{$\texttt{append}(\vec{v}, v_i)$} \COMMENT{append threshold for bucket $i$}
            \STATE{$l \leftarrow l + l_i$}  \COMMENT{accumulate loss from bucket $i$}
        \ENDFOR
        \IF {$ l < l^{min} $}
            \STATE{$l^{min} \leftarrow l \text{, } j^{min} \leftarrow j \text{, } \vec{v}^{min} \leftarrow \vec{v} $}
            \COMMENT{new best split}
        \ENDIF
    \ENDFOR \label{line:dimEvalEnd}

    \LINECOMMENT{create new buckets using chosen split}
    \STATE{$\bm{\mathcal{B}} \leftarrow $ [ ]} \label{line:splitBucketsStart}
    \FOR{$i \leftarrow 1 \textbf{ to } 2^{t-1} $}
        \STATE{ $\mathcal{B}_{below} \text{, } \mathcal{B}_{above} \leftarrow \texttt{apply\_split}(v^{min}_i \text{, } \mathcal{B}^{t-1}_i) $}
        % \STATE{$\texttt{append}(\bm{\mathcal{B}}, [\mathcal{B}_{below}, \mathcal{B}_{above}])$}
        \STATE{$\texttt{append}(\bm{\mathcal{B}}, \mathcal{B}_{below})$}
        \STATE{$\texttt{append}(\bm{\mathcal{B}}, \mathcal{B}_{above})$}
    \ENDFOR \label{line:splitBucketsEnd}
    \STATE{\textbf{return } $\bm{\mathcal{B}} \text{, } l^{min} \text{, } j^{min} \text{, } v^{min}$}
\end{algorithmic}
% \vspace*{4mm}
\end{algorithm}

\subsection{Optimizing the Prototypes}
% \vspace{-.5mm}
% ------------------------------------------------

At this point, we have a complete algorithm. We could simply drop our hash-based encoding function into PQ and approximate matrix products. However, we contribute two additional enhancements: a means of optimizing the prototypes with no runtime overhead, and a means of quickly summing low-bitwidth integers.

% First, we introduce a means of optimizing the prototypes given only the matrix $\tilde{\A}$.
Several works propose prototype or table optimizations based on knowledge of $\B$ \cite{pairq,optimizedDists}, and others optimize them at the expense of slowing down the function $g(\cdot)$ \cite{cq,scq}. In contrast, we introduce a method that does not do either of these. The idea is to choose prototypes such that $\tilde{\A}$ can be reconstructed from its prototypes with as little squared error as possible---this improves results since less error means that less information about $\tilde{\A}$ is being lost.
% Our insight is that
% \begin{enumerate}
%     \item The quality of the matrix product approximation increases when the quantization error between each row of $\A$ and its corresponding prototypes decreases.
%     \item This quantization error is equal to the difference between $\A$ and the ``reconstructed'' $\A$ formed by replacing each row with the sum of its assigned prototypes.
%     % reconstruct $\tilde{\A}$ from each row's prototype assignments.
%     \item Reconstruction can be formulated as a linear regression problem. The prototypes are a good but not optimal solution to this problem.
% \end{enumerate}
% one can exploit mutual information between encodings in different subspaces to more accurately approximate the data distribution; furthermore, the quality of this

Let $\mat{P} \in \R^{KC \times D}$ be a matrix whose diagonal blocks of size $K \times |\mathcal{J}^{(c)}|$ consist of the $K$ learned prototypes in each subspace $c$. The training matrix $\tilde{\A}$ can be approximately reconstructed as $\tilde{\A} \approx \mat{G}\mat{P}$,
% \begin{align}
%     \tilde{\A} \approx \mat{G}\mat{P}
% \end{align}
where $\mat{G}$ serves to select the appropriate prototype in each subspace. Rows of $\mat{G}$ are formed by concatenating the one-hot encoded representations of each assignment for the corresponding row of $\tilde{\A}$. For example, if a row were assigned prototypes $\langle$\texttt{3 1 2}$\rangle$ with $K = 4$, $C = 3$, its row in $\mat{G}$ would be $\langle$\texttt{0010 1000 0100}$\rangle \in \R^{12}$. Our idea is to optimize $\mat{P}$ conditioned on $\mat{G}$ and $\tilde{\A}$. This is an ordinary least squares problem, and we solve it with ridge regression:
% Essentially, $\mat{G}$ serves to select the appropriate prototype in each subspace.
\begin{align}
% \vspace*{-1mm}
    \mat{P} \triangleq (\mat{G}^\top \mat{G} + \lambda \mat{I})^{-1} \mat{G}^\top \tilde{\A}.
% \vspace*{-.5mm}
\end{align}
One could obtain better performance by cross-validating to find $\lambda$, but for simplicity, we fix $\lambda = 1$. %, which corresponds to adding the minimum sensible pseudocount to the assignment cooccurrence matrix $\mat{G}^\top \mat{G}$.

% In short, the initial prototypes are just used to produce an assignment matrix $\mat{G}$, and it is from this assignment matrix that the final prototypes are derived.
This procedure allows the prototypes to be nonzero outside of their original subspaces. Because of our hashing procedure, we avoid the dramatically increased overhead faced by other methods with non-orthogonal prototypes (c.f. \cite{otq,aq,cq,grvq,lsq,stackedQuantizers}).

% ------------------------------------------------
% \vspace{-1.5mm}
\subsection{Fast 8-bit Aggregation, $\bm{f(\cdot,\cdot)}$} \label{sec:aggregate}
% \vspace{-.5mm}
% ------------------------------------------------
Let $\mat{T} \in \R^{M \times C \times K}$ be the tensor of lookup tables for all $M$ columns of $\B$.
Given the encodings $\mat{G}$, the function $f(\cdot,\cdot)$ is defined as
\begin{align}
\vspace*{2mm}
    % f(g(\A), h(\B))_{n,m} \triangleq \sum_{c=1}^C \sum_{k=1}^K \mat{T}^{q}_{m,c,k} I\{\mat{G}_{n,c} = k \}
    f(g(\A), h(\B))_{n,m} \triangleq \sum_{c=1}^C \mat{T}_{m,c,k} \cs\text{ } k = g^{(c)}(\a_n).
\end{align}
Because the entries of $\mat{T}$ are stored as 8-bit values, exact summation requires immediately upcasting each looked-up entry to 16 bits before performing any addition instructions \cite{bolt}. This not only imposes overhead directly, but also means that one must perform 16-bit additions, which have half the throughput of 8-bit additions.

We propose an alternative that sacrifices a small amount of accuracy for a significant increase in speed. Instead of using \textit{addition} instructions, we use \textit{averaging} instructions, such as \texttt{vpavgb} on x86 or \texttt{vrhadd} on ARM. While non-saturating additions compute $(a + b) \text{ \% } 256$, these instructions compute $(a + b + 1) / 2$. This means that they lose information about the low bit instead of the high bit of the sum. We estimate the overall mean by averaging pairs of values, then pairs of pairs, and so on. We refer the reader to Appendix~\ref{sec:aggregateAnalysis} for details.

% One could reduce all $C$ values this way, but we find that one obtains a better speed-accuracy tradeoff by computing the average of blocks of $U$ values and then upcasting to obtain exact sums of these averages. Multiplying this sum of averages by $U$ and adding in a bias correction term gives one the overall estimate of the sum. One could tune $U$ for a particular problem and hardware, but we simply set $U = 16$ in all our experiments.

The challenging part of this approach is computing the bias in the estimated sum in order to correct for it. We prove in Appendix~\ref{sec:aggregateAnalysis} that this bias has a closed-form solution under the realistic assumption that the low bits are equally likely to be 0 or 1.

% % ------------------------------------------------
% \vspace{-1.5mm}
% \subsection{Complexity}
% \vspace{-.5mm}
% % ------------------------------------------------

% Our encoding function $g(\A), \A \in \R^{N \times D}$ has complexity $\Theta(NC)$, since it does a constant amount of work per row per codebook. Our table creation function $h(\B), \B \in \R^{D \times M}$ has complexity $\Theta(MKCD)$, since it must compute the inner product between each column of $\B$ and $KC$ prototypes of length $D$. This is a factor of $C$ worse than PQ since we do not require the prototypes for different codebooks to have disjoint nonzero indices. However, as discussed in section~\ref{sec:problemStatement}, this reduction in the speed of $h(\cdot)$ is not a concern. Finally, the complexity of our aggregation function $f(\cdot)$ is $\Theta(NCM)$, since it performs $C$ table lookups for each of $M$ output columns and $N$ output rows. This means our overall algorithm has complexity $\Theta(MC(KD + N))$, which reduces to $\Theta(NCM)$ since we fix $K = 16$, and our problem statement requires $N \gg D$.

% ------------------------------------------------
% \vspace{-1mm}
\subsection{Theoretical Guarantees} \label{sec:maddnessMainThm}
% \vspace{-1.5mm}
% ------------------------------------------------

% All of the guarantees for \textsc{Bolt}\text{} \cite{bolt} also apply to \ours, modulo a small amount of additional error from averaging integers rather than summing exactly. This follows from \textsc{Bolt}'s guarantees depending only upon the quantization errors, rather than the method used to obtain them.
% Beyond these existing guarantees,

Our central theoretical result is a generalization guarantee for the overall approximation error of \ours, stated below. See Appendix~\ref{sec:maddnessMath} for a proof and additional analysis, including a discussion of algorithmic complexity. Besides this main guarantee, we also inherit all of the guarantees for Bolt \cite{bolt}, modulo a small amount of additional error from averaging integers rather than summing exactly. This follows from Bolt's guarantees depending only on the quantization errors, rather than the method used to obtain them.

\begin{theorem}[Generalization Error of \ours] \label{thm:maddness}
Let $\Dcal$ be a probability distribution over $\R^D$ and suppose that \oursp is trained on a matrix $\tilde{\A} \in \R^{N \times D}$ whose rows are drawn independently from $\Dcal$ and with maximum singular value bounded by $\sigma_A$. Let $C$ be the number of codebooks used by \oursp and $\lambda > 0$ be the regularization parameter used in the ridge regression step. Then for any $\b \in \R^D$, any $\a \sim \Dcal$, and any $0 < \delta < 1$, we have with probability at least $1 - \delta$ that
\begin{align}
    \begin{split}
    \E_{\Dcal}[&\Lcal(\a, \b)] \le \E_{\tilde{\A}}[\Lcal(\a, \b)] + \\
    &\frac{C \sigma_A \norm{\b}_2}{2 \sqrt{\lambda}} \left(
        \frac{1}{256} +
        \frac{
            8 +
            \sqrt{
                % C (4\ceil{\log_2(D)} + 256) \log{2} -\log{\delta}
                \nu(C, D, \delta)
            }
        }{\sqrt{2n}}
    \right)
    \end{split}
% \end{equation}
\end{align}
% where
% where $\Lcal(\a, \b) \triangleq |\a^\top \b - \alpha f(g(\a), h(\b)) - \mat{\beta}|$ (c.f., Equation~\ref{eq:objective}).
where $\Lcal(\a, \b) \triangleq |\a^\top \b - \alpha f(g(\a), h(\b)) - \mat{\beta}|$, $\alpha$ is the scale used for quantizing the lookup tables, $\mat{\beta}$ is the constants used in quantizing the lookup tables plus the debiasing constant of Section~\ref{sec:aggregate}, and
\begin{align}
\vspace*{-2mm}
    \nu(C, D, \delta) \triangleq C (4\ceil{\log_2(D)} + 256) \log{2} -\log{\delta}.
\vspace*{-2mm}
\end{align}
% The constants $\alpha$ and $\mat{\beta}$ are the values used for quantizing the lookup tables, plus the debiasing constant of Section~\ref{sec:aggregate}.
% with $\alpha$ and $\mat{\beta}$ the values used for quantizing the lookup tables (plus the debiasing constant of Section~\ref{sec:aggregate}).
\end{theorem}

% See Appendix~\ref{sec:maddnessMath} for a proof and additional analysis, including a discussion of algorithmic complexity. Note that we also inherit all of the guarantees of \textsc{Bolt}\text{} \cite{bolt}, modulo a small amount of additional error from averaging integers rather than summing exactly.
% This follows from \textsc{Bolt}'s guarantees depending only upon the quantization errors and not the method used to obtain them.

% The intuition behind the proof is that the overall approximation decomposes into three parts: lookup table quantization error, a learned hash function, and a linear model. The table quantization error is provably small. The learned hash function can be expressed using finitely many bits, and therefore has a finite hypothesis space over which we can union bound. The linear model's generalization error can be bounded using known theorems \cite{kakadeLinear} based on Rademacher complexity \cite{rademacherOrig}. However, using these theorems requires bounds on the norms of linear combinations of the prototypes, which are difficult to obtain. % A key step in bounding these norms is obtaining a guarantee regarding the singular values of weight matrices trained through ridge regression, which may be of independent interest.

%================================================================
% \vspace{-2mm}
\section{Experiments} \label{sec:results}
% \vspace{-.5mm}
%================================================================

To assess \ours's effectiveness, we implemented both it and existing algorithms in C++ and Python. All of our code and raw numerical results are publicly available at \texttt{https://smarturl.it/Maddness}. All experiments use a single thread on a Macbook Pro with a 2.6GHz Intel Core i7-4960HQ processor. Unless stated otherwise, all timing results use five trials, with each trial reporting the fastest among 20 executions. We use the best, rather than the average, since this is standard practice in performance benchmarking and is robust to the purely additive noise introduced by competing CPU tasks. Standard deviations are shown for all curves as shaded areas.
% , but are often too small to see.
Since training can be performed offline and all methods except SparsePCA \cite{sparsePCA} train in at most a few minutes, we omit profiling of training times. We also do not profile the time to preprocess $\B$, since 1) this time is inconsequential in most cases, and 2) $\B$ is fixed and could be processed offline in all the problems we consider.
In order to avoid implementation bias, we build upon the source code provided by \citet{bolt}\footnote{https://github.com/dblalock/bolt}, which includes highly tuned implementations of many algorithms to which we compare.

% The metrics and baselines shown below are only a subset of those evaluated. In particular, we also measured squared errors, correlations, cosine similarities, etc, in addition to classification accuracies, but omit these since they are similar. See Appendix~\ref{sec:experimentDetails} for more results.

% Additional details about data cleaning, hyperparameter tuning, and other minutiae can be found in Appendix~\ref{sec:experimentDetails}. We do not need to tune any hyperparameters for our own method, but to ensure that other methods are not hindered by insufficient hyperparameter tuning, we allow them to tune their parameters on the test data and select the best results at a given speed \textit{post-hoc}.
We do not need to tune any hyperparameters for \ours, but we do take steps to ensure that other methods are not hindered by insufficient hyperparameter tuning. Concretely, we sweep a wide range of hyperparameter values and allow them to cherry-pick their best hyperparameters on each test matrix.
Further details regarding our experimental setup and choices (e.g., use of a single thread) can be found in Appendix~\ref{sec:experimentDetails}.

% data cleaning, hyperparameter tuning, reasons for experimental choices (e.g., using one thread), and other aspects of our experimental setup can be found in Appendix~\ref{sec:experimentDetails}.

% Also note that we compared to many methods not shown here, such as numerous randomized methods and variations of the Frequent Directions method \cite{liberty_simple_2012, ghashami_frequent_2016}. We report only those shown here since they performed the best.

% Because nearly all existing work on approximate matrix multiplication either focuses on special cases that do not satisfy our problem definition \cite{quickerAdc, pq, opq} or synthetic matrices, there is not a clear set of benchmark matrix multiply tasks to use. We therefore propose a collection of tasks that we believe are both reproducible and representative of many real-world matrices. To the best of our knowledge, our experiments use over an order of magnitude more matrices than any previous study.

% ------------------------------------------------
\vspace{-1mm}
\subsection{Methods Tested}
\vspace{-.5mm}
% ------------------------------------------------
% Recall that most baselines take the form of selecting a matrix $\V \in \R^{D \times d}, d < D$ such that $\A \B \approx (\A \V) (\V^\top \B)$. Here $d$ is a free parameter that adjusts the quality vs speed tradeoff. We therefore characterize these methods by how they set $\V$.
Recall that most baselines take the form of selecting a matrix $\V \in \R^{D \times d}, d < D$ such that $\A \B \approx (\A \V) (\V^\top \B)$. Here $d$ is a free parameter that adjusts the quality versus speed tradeoff. We therefore characterize most of these methods by how they set $\V$.
\vspace{-2.5mm}
% \begin{itemize}\itemsep-1mm
% \begin{itemize}\itemsep0mm
\begin{itemize}\itemsep.8mm
% \begin{itemize}[nosep]
    \item \textbf{PCA}. Set $\V$ equal to the top principal components of $\tilde{\A}$.
    \item \textbf{SparsePCA} \cite{sparsePCA}. Set $\V = \argmin_{\V} \min_{\mat{U}} \frac{1}{2 N_{train}} \norm{ \tilde{\A} - \mat{U}\mat{V}^\top }^2_F + \lambda \norm{\V}_1$, where $\mat{U}^\top \mat{U} = \mat{I}$. This is not the only dictionary learning formulation referred to as SparsePCA \cite{spcaSurvey1,spcaSurvey2}, but it is a good representative and is the only one with support in a major Python library.%\footnote{https://scikit-learn.org/stable/modules/generated/sklearn.decomposition.SparsePCA.html}.

    % Because SparsePCA requires the tuning of both $d$ and $\lambda$ for each matrix, we allowed it tune these parameters \textit{on the test set} to ensure that insufficient hyperparameter tuning did not hamper its performance. See Appendix~\ref{sec:experimentDetails} for more information.

    % First, for any given $d$ value and level of sparsity, we report
    % For each matrix product, we tried $\lambda \in 2^i, i \in \{-5, -4, -3, -2, -1, 0, 1, 2, 3\}$ for each matrix product. To ensure that our results err on the side of optimism (and also speed up our experiments), we report the best
    % selected the lambda \textit{post-hoc}---i.e., we cherrypicked the best results on the test set rather than cross-validating,
    \item \textbf{FastJL} \cite{fastjl}. $\V$ is set to Rademacher random variables composed with a Fast Hadamard Transform (FHT). For simplicity, we exclude the FHT in the timing.
    \item \textbf{HashJL} \cite{hashjl}. $\V$ is zero except for a $\pm 1$ in each row, with both sign and position chosen uniformly at random.
    % \item OSNAP \cite{osnap}. The OSNAP sketch is essentially a collection of $s$ HashJL sketches, each of dimensionality $d / s$. Following \cite{iSVD}, we set $s = 4$. We also tried other values of $s$ and found that other values did not perform significantly better (except $s=1$, which reduces to HashJL).
    % \item \textbf{RandGauss} \cite{lshOrig,E2LSH}. The entries of $\V$ are drawn i.i.d. from $\mathcal{N}(0, \frac{1}{D}\mat{I})$.
    \item \textbf{ScalarQuantize}. The matrices are not projected, but instead linearly quantized to eight bits such that the smallest and largest entries map to
    % the smallest and largest values expressible with eight bit integers.
    either 0 and 255 or -128 and 127, as appropriate.
    We use FBGEMM \cite{fbgemm} to perform the quantized matrix multiplies. We neglect the time required to convert from other formats to eight bits, reflecting the optimistic scenario in which the matrices are already of the appropriate types.
    % \item OrthoGauss \cite{superbitLSH}. The entries of $\V$ initialized as in RandGauss, and then $\V$ is set to the $\mat{Q}$ matrix in a QR decomposition of a RandGauss matrix.
    % \item Rademacher \cite{rademacherJL}. The entries of $\V$ are i.i.d. Rademacher random variables scaled by $\frac{1}{\sqrt{D}}$. TODO is that the right scale?
    % \item \textbf{Bolt} \cite{bolt}. Bolt is an extension of PQ that uses quantized lookup tables and a reduced number of codebooks to obtain $10\times$ speedups over traditional PQ. It is the most similar method to our own, differing only in the encoding function, the use of averaging instead of upcasting, and the optimization of centroids. % It is not optimized for small $M$, however, effectively performing a preliminary matrix product with $M=16$ as a preprocessing step.
    \item \textbf{Bolt} \cite{bolt}. Bolt
    % is an extension of PQ that uses quantized lookup tables and a reduced number of codebooks to obtain $10\times$ speedups over traditional PQ. It
    is the most similar method to our own, differing only in the encoding function, the use of averaging instead of upcasting, and the optimization of centroids. % It is not optimized for small $M$, however, effectively performing a preliminary matrix product with $M=16$ as a preprocessing step.
    % As discussed in Section~\ref{sec:relatedWork},
    \item \textbf{Exact Multiplication}. We simply compute the matrix product $\A\B$ using a modern BLAS implementation.
    \item \textbf{\ours-PQ}. A handicapped version of \oursp without the prototype optimization step. The gap between \oursp and \ours-PQ is the gain from optimizing the prototypes.
\end{itemize}
\vspace{-2mm}
% In addition to these baselines, we test two variations of our method:
% We also test two variations of our method:
% We also test one variation of our method:
% \vspace{-2.5mm}
% % \begin{itemize}\itemsep-.5mm
% \begin{itemize}\itemsep0mm
% % \begin{itemize}[nosep]
%     % \item \textbf{\ours}. The algorithm described in Section~\ref{sec:method}.
%     \item \textbf{\ours-PQ}. A handicapped version of \oursp without the prototype optimization step. The gap between \oursp and \ours-PQ is the gain from optimizing the prototypes.
% \end{itemize}
% \vspace{-2.5mm}
% compared to many methods not shown here, such as numerous randomized methods and variations of the Frequent Directions method \cite{liberty_simple_2012, ghashami_frequent_2016}. We report only those shown here since they performed the best.
We also compared to many additional methods (see Appendix~\ref{sec:experimentDetails}), but omit their results since they were not competitive with those listed here.

% ------------------------------------------------
\vspace{-1mm}
\subsection{How Fast is \ours?}
\vspace{-.5mm}
% ------------------------------------------------

We begin by profiling the raw speed of our method. In Figure~\ref{fig:encodeSpeed}, we time the $g(\A)$ functions for various vector quantization methods. The $\A$ matrices have $2^{14}$ rows and varying numbers of columns $D$. Following \citet{bolt}, we also vary the number of codebooks $C$, profiling 8-, 16-, and 32-byte encodings. We measure in bytes rather than codebooks since PQ and OPQ use eight bits per codebook while Bolt and \oursp use four.

\oursp is up to two orders of magnitude faster than existing methods, and its throughput increases with row length. This latter property is because its encoding cost per row is $O(C)$ rather than $O(D)$.
%Because its encoding cost is $O(C)$ rather than $O(D)$ \ours's throughput grows with row length.
\begin{figure}[h]
\begin{center}
\includegraphics[width=\linewidth]{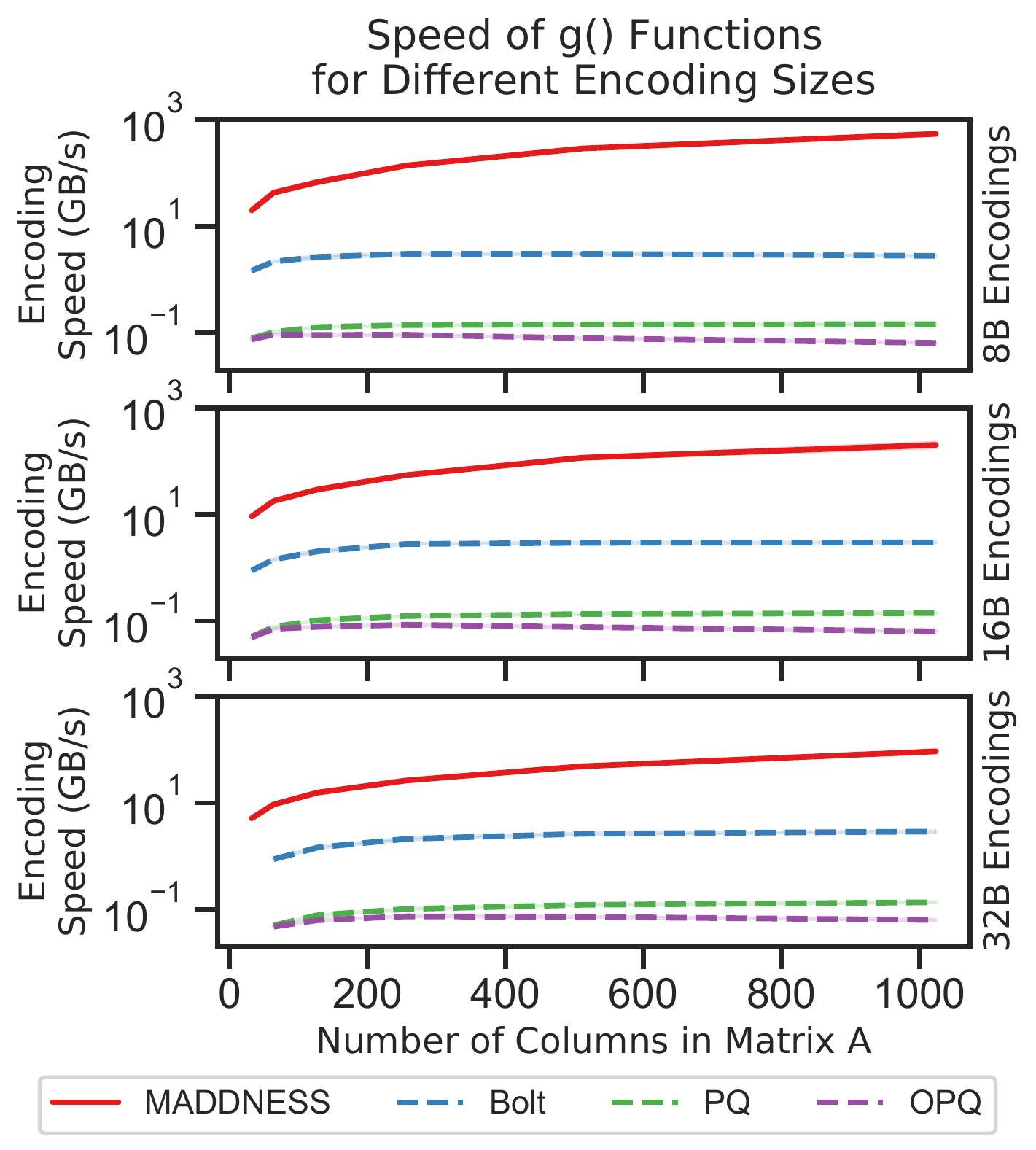}
\caption{\oursp encodes the $\A$ matrix orders of magnitude more quickly than existing vector quantization methods.}
\label{fig:encodeSpeed}
\end{center}
\end{figure}

% \oursp is up to two orders of magnitude faster than existing methods. \ours's throughput grows with row length because its encoding cost is $O(C)$ rather than $O(D)$. % \oursp is not only faster than existing methods, but faster than the machine's memory bandwidth of 15GB/s. This is possible because it only reads $O(C)$ columns, and $C$ can be lower than $D$.
% \vspace{-2mm}
\vspace{-1mm}
We also profile the speed of our aggregation function $f(\cdot, \cdot)$ using the same baselines as \citet{bolt}. As Figure~\ref{fig:scanSpeed} shows, our average-based, matrix-aware aggregation is significantly faster than the upcasting-based method of Bolt, its nearest rival.
\vspace{2mm}
%Though less dramatic a speedup than for the encoding function, we see that our average-based aggregation is significantly faster than the upcasting-based method of Bolt, its nearest rival.

\begin{figure}[h]
\begin{center}
\includegraphics[width=\linewidth]{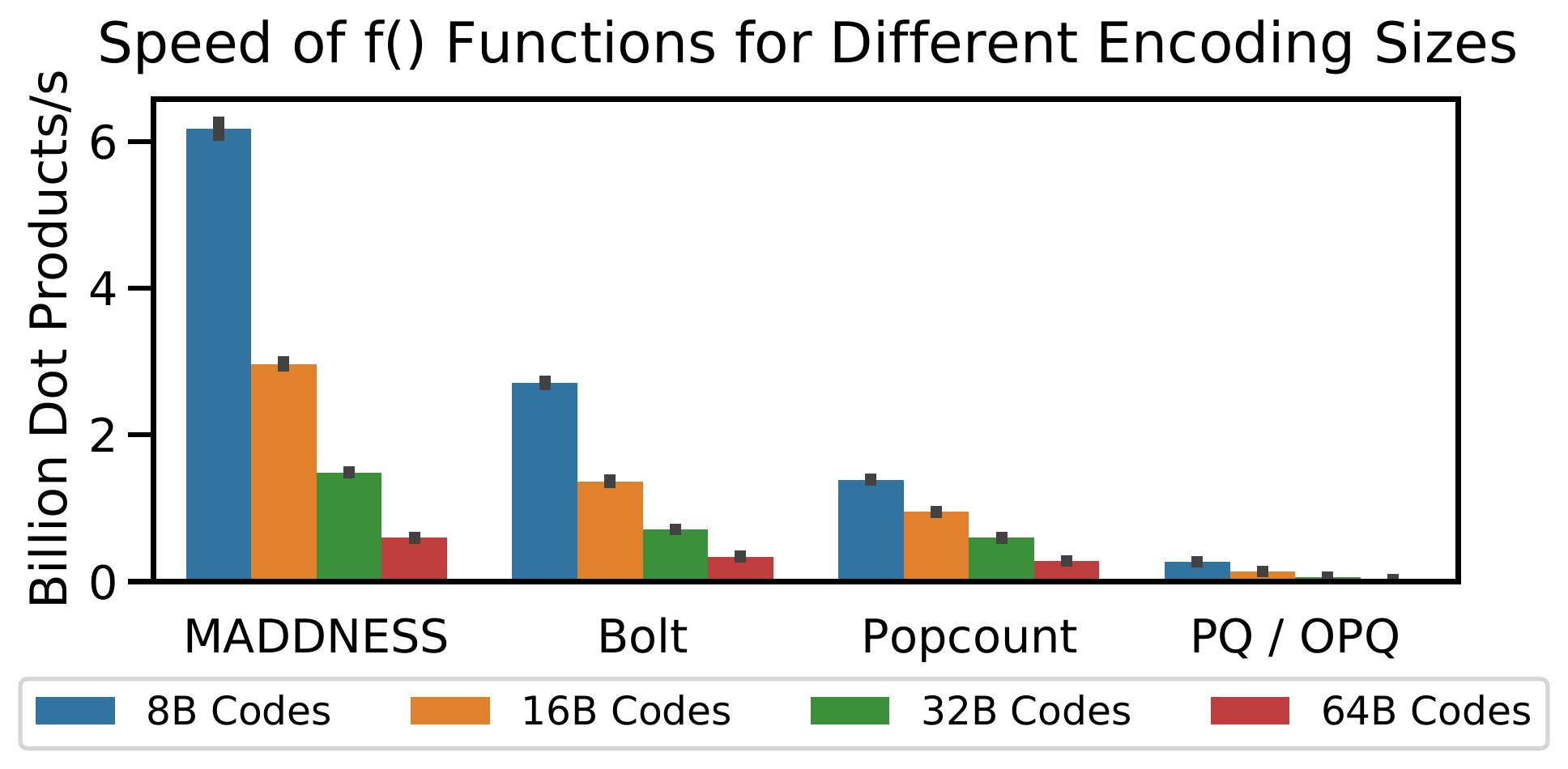}
\caption{Given the preprocessed matrices, \oursp computes the approximate output twice as fast as the fastest existing method.}
\label{fig:scanSpeed}
\end{center}
\end{figure}

% ------------------------------------------------
% \vspace{-1.5}
% \vspace{-2mm}
\subsection{Softmax Classifier}
% \vspace{-.5mm}
% \vspace{-1mm}
% ------------------------------------------------

As described in Section~\ref{sec:intro}, we approximated linear classifiers on the widely used CIFAR-10 and CIFAR-100 datasets \cite{cifarDsets}. The classifiers use as input features the 512-dimensional activations of open-source, VGG-like neural networks trained on each dataset \cite{cifarVgg}. The matrices $\A$ are the $10000 \times 512$-dimensional floating point activations for the full test sets, and the matrices $\B$ are each network's final dense layer. The $50000 \times 512$-dimensional activations from the training set serve as the training matrices $\tilde{\A}$.
As shown in Figure~\ref{fig:cifar}, \oursp significantly outperforms all existing methods, achieving virtually the same accuracy as exact multiplication more than an order of magnitude faster.

\begin{figure}[h]
\begin{center}
\includegraphics[width=\linewidth]{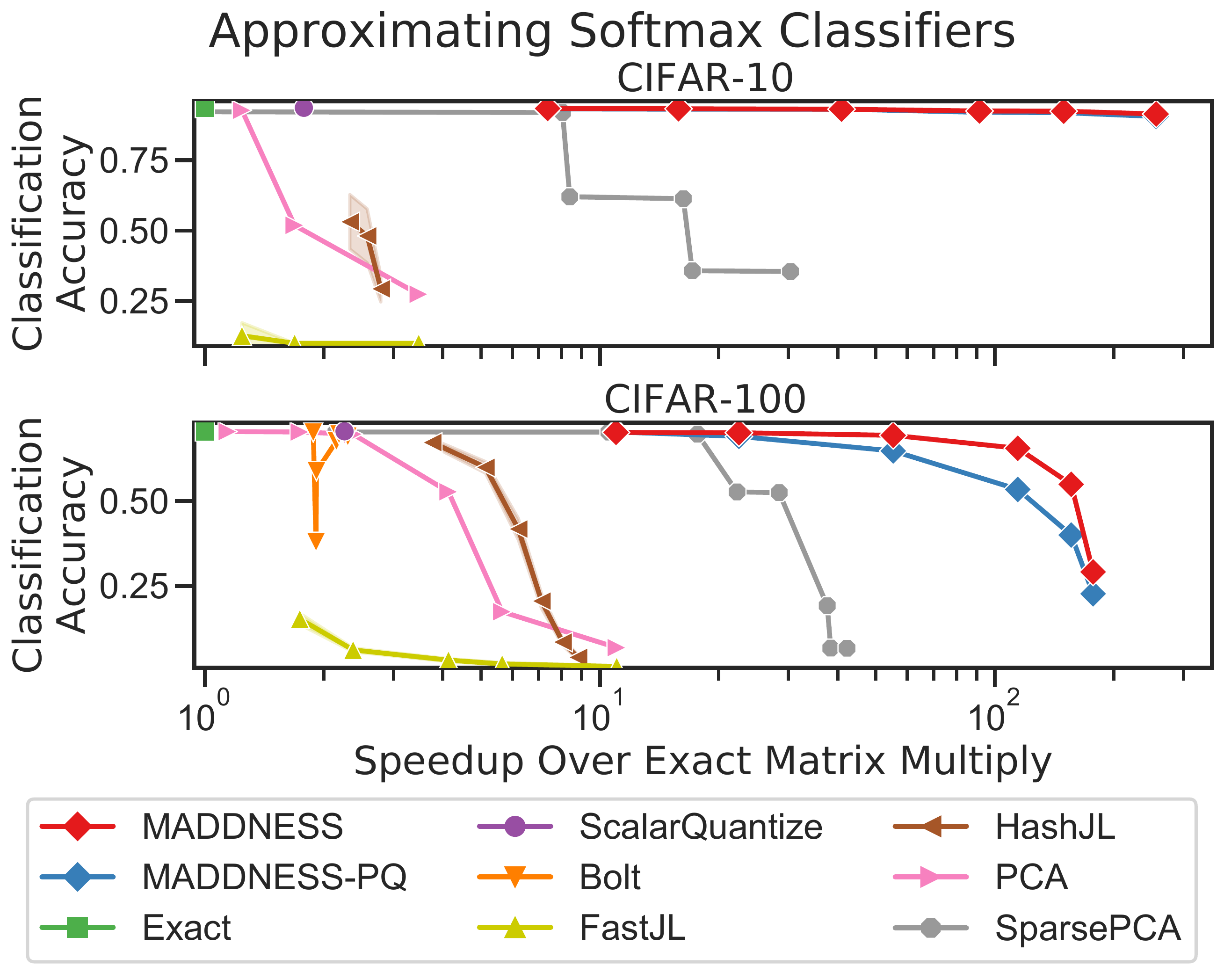}
\caption{\oursp achieves a far better speed-accuracy tradeoff than any existing method when approximating two softmax classifiers.}
\label{fig:cifar}
\end{center}
\end{figure}

Moreover, our method achieves this performance despite having worse support from the hardware. More precisely, it obtains speedups much smaller than the level of compression it provides. For example, the third points from the right in both plots correspond to speedups of roughly $100\times$. However, they are compressing each $512 \times 4\text{B} = 2048\text{B}$ row of the input down to a mere 4B, a savings of $512\times$ (sizes not shown in figure). \textit{If the hardware could lookup-accumulate as many bytes per cycle as it can multiply-accumulate, our method could be over $4\times$ faster}. Combined with the fact that multiplexers require many fewer transistors than multipliers, this suggests that a hardware implementation of our method might offer large efficiency gains compared to existing accelerators.

% ------------------------------------------------
% \vspace{-2mm}
\subsection{Kernel-Based Classification}
% \vspace{-.5mm}
% \vspace{-1mm}
% ------------------------------------------------

To assess the efficacy of our method on a larger and more diverse set of datasets than CIFAR-10 and CIFAR-100, we trained kernel classifiers on the datasets from the UCR Time Series Archive \cite{UCRArchive2018}. To enable meaningful speed comparison across datasets, we resampled the time series in all datasets to the median length and obtained the matrix $\B$ for each dataset by running Stochastic Neighbor Compression \cite{snc} on the training set with an RBF kernel of bandwidth one.
% We set the number of returned neighbors to 128 (results with 64 and 256 were similar).
We approximate the Euclidean distances used by the kernel via the identity $\norm{\vec{x} - \vec{y}}_2^2 = \norm{\vec{x}}_2^2 - 2\vec{x}^\top \vec{y} + \norm{\vec{y}}_2^2$, which consists only of dot products.
This is not the state-of-the-art means of classifying time series, but it does yield fixed-sized matrices and is representative of several modern techniques for constructing highly efficient classifiers \cite{snc,dsnc,bnc,protonn}.
% This setup also complements the CIFAR results by being an extremely
Because Stochastic Neighbor Compression optimizes the classifiers to avoid redundancy, this task is quite difficult.
% This is because 1) Stochastic Neighbor Compression has already optimized the classifier to avoid redundancy, and 2) the distances between time series are often far smaller than their $L_2$ norms, allowing even small approximation errors to change predictions.
% since the distances between time series are often far smaller than their euclidean lengths; this property means that even small amounts of approximation error are sufficient to change predictions. It is also difficult because the $\B$ matrix has been optimized to avoid redundancy, so any acceleration
% \vspace{-.5mm}

As shown in Figure~\ref{fig:ucr}, \oursp is significantly faster than alternatives at a given level of accuracy. % on this difficult task.
A counter-intuitive result, however, is that optimization of the prototypes occasionally reduces accuracy---see the red line dipping below the blue one in the lowest subplot. Since the optimization strictly increases the expressive power, we believe that this is a product of overfitting and could be corrected by not fixing $\lambda = 1$ in the ridge regression.
% were we to tune the ridge regression's parameter (instead of always using $\lambda = 1$).
% This subplot also shows that even \oursp cannot satisfy the most stringent accuracy requirements on many datasets, since the \oursp curves never approach 1.
% This subplot also reveals that the most stringent degredation requirements can sometimes only be met using PCA at a low speedup (and not \ours), since this is the only curve approaching $1$.
% , at high accuracy preservation thresholds,  % However as indicated by \oursp never approaching a fraction of 1 on the bottom subplot, \oursp does not consistently preserve the full accuracy. This suggests that, for difficult classification problems in which almost no accuracy can be sacrificed, \outsp is often not the best choice. The only methods that reliably preserve nearly the full original accuracy are PCA with a small speedup and, with certain parameter settings, Bolt.

\begin{figure}[h]
\begin{center}
\includegraphics[width=\linewidth]{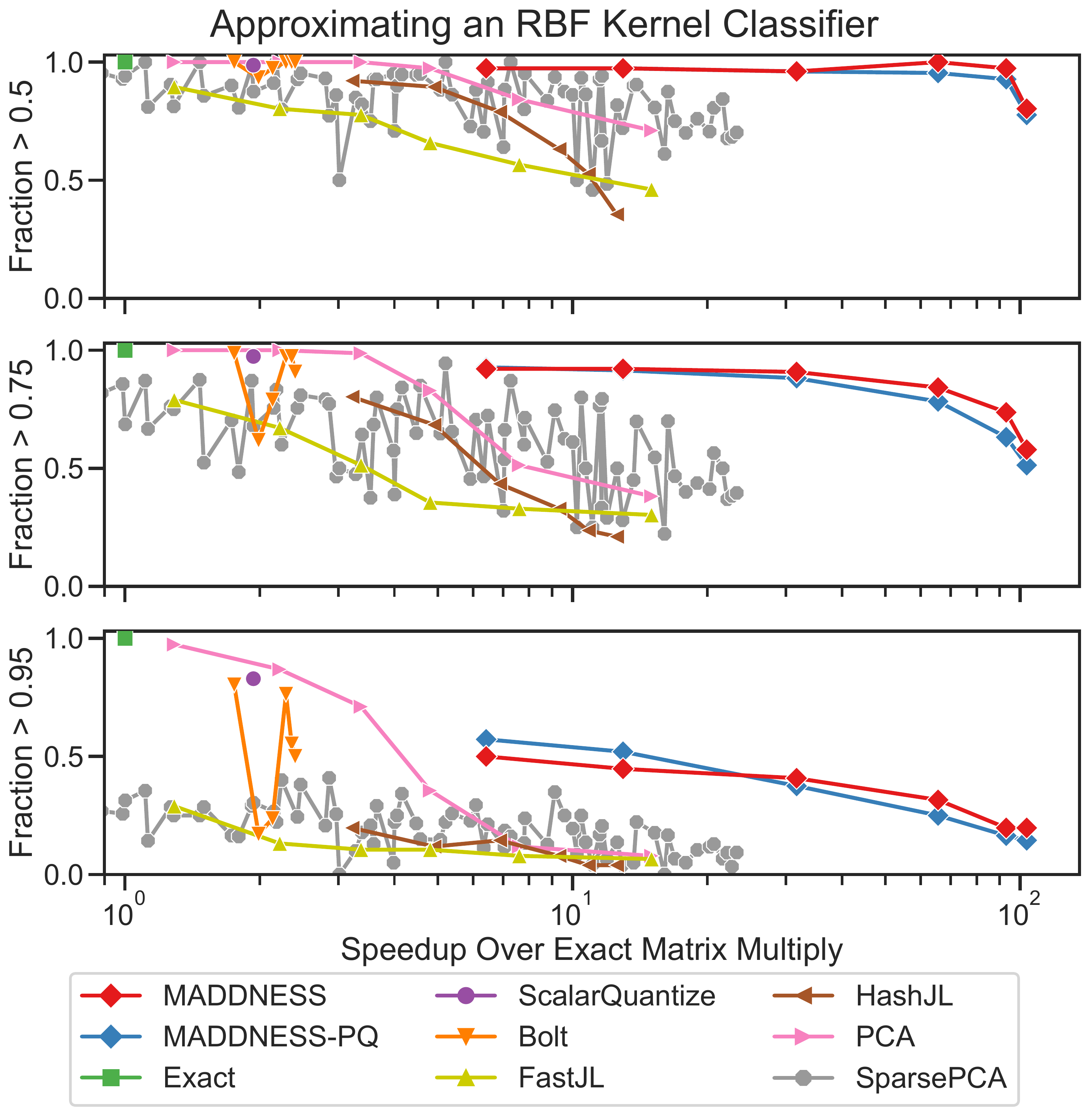}
\caption{Fraction of UCR datasets for which each method preserves a given fraction of the original accuracy.
% , versus the method's degree of speedup
\oursp enables much greater speedups for a given level of accuracy degradation.}
\label{fig:ucr}
\end{center}
\end{figure}

% ------------------------------------------------
\vspace{-1mm}
\subsection{Image Filtering}
\vspace{-.5mm}
% ------------------------------------------------

To test the extreme limits of \ours, we benchmarked the various techniques' ability to apply small filters to images (after an offline im2row transform to reduce the task to matrix multiplication). This task is extreme in that $D$ and $M$ are tiny, affording almost no opportunity to amortize preprocessing costs. As representative example filters, we chose $3 \times 3$ Sobel kernels and $5 \times 5$ Gaussian kernels. These are common high-pass and low-pass filters, respectively. We took the first 10 images from the first 50 classes of the Caltech101 dataset \cite{caltech} as a single training set, and the first 10 images from the remaining 51 classes as 510 test sets. We constructed the $\A$ matrices by extracting each patch of each image as one row. The $\B$ matrices have two columns, corresponding to one pair of Sobel or Gaussian filters (since using these filters in pairs is common). %This is not the most efficient way to perform image filtering since it does not exploit the structure of the convolution operation or the filters, but it does serve as a useful benchmark.
 %See Appendix~\ref{sec:experimentDetails} for additional details.
We report the normalized mean-squared error (NMSE), defined as $\norm{\mat{\hat{C}}_{i,j} - \A\B}^2_F / \norm{\A\B}^2_F$, where $\hat{\mat{C}}$ is the method's estimate of $\A\B$. An NMSE of 0 is perfect and an NMSE of 1 corresponds to always predicting 0.

In Figure~\ref{fig:caltech}, we see that it is only \oursp that offers any advantage over exact matrix products. This is likely because two columns afford almost no time to preprocess $\A$; indeed, rival vector quantization methods cannot logically do less work than brute force in this setting, and dense linear methods can only save work by embedding rows of $\A$ in one-dimensional space.
\oursp performs much worse on the high-pass filters (top) than the low-pass filters (bottom). This is likely because the former produce outputs with variance that is orders of magnitude lower than that of the original image, making the NMSE denominator tiny.
% this makes the NMSE denominator, $\norm{\A\B}^2_F$, tiny.  % relative to $\norm{\A}_F$ and $\norm{\B}_F$.
% this means that the NMSE denominator, $\norm{\A\B}^2_F$, is tiny.  % relative to $\norm{\A}_F$ and $\norm{\B}_F$.

\begin{figure}[h]
\begin{center}
\includegraphics[width=\linewidth]{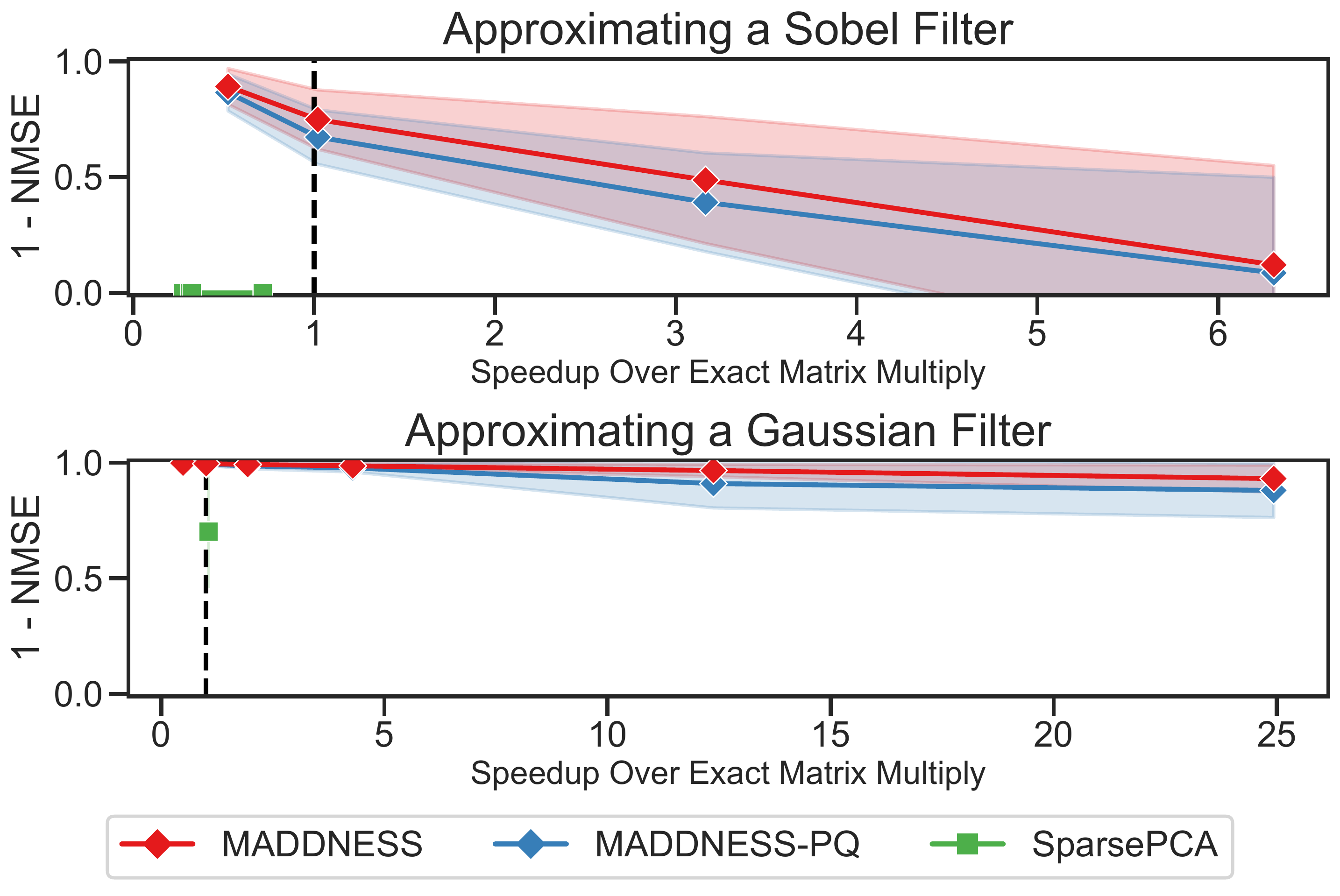}
\caption{Despite there being only two columns in the matrix $\B$, \oursp still achieves a significant speedup with reasonable accuracy. Methods that are Pareto dominated by exact matrix multiplication on both tasks are not shown; this includes all methods but \oursp and SparsePCA.} % It is, however, much more effective for approximating a 5x5 low-pass filter (bottom) than a 3x3 high-pass filter (top).
%Only \oursp is not Pareto dominated by exact matrix multiplication (.}
\label{fig:caltech}
\end{center}
\end{figure}

% % ------------------------------------------------
% \subsection{Summary}
% % ------------------------------------------------

% Our method without \textit{any} hyperparameter tuning outperforms its nearest rival with all its hyperameters tuned \textit{on the test set}.

% ================================================================
% \vspace*{-5mm}
\section{Discussion and Conclusion} \label{sec:limitations}
% \vspace{-.5mm}
% ================================================================

% \section{Conclusion}

%
% \vspace{-1mm}

Because our work draws on a number of different fields but does not fit cleanly into any of them, it is useful to discuss what we have and have not demonstrated, as well as possible implications and extensions of our work.

% ================================================================

Our main empirical finding is that our proposed method, \ours, achieves order-of-magnitude speedups compared to existing AMM methods and up to two-order-of-magnitude speedups compared to the dense baseline. It also compresses matrices by up to three orders of magnitude. These results are evaluated on a CPU, and are obtainable only when there is a training set for one matrix. We also claim superior performance only when one matrix is larger than the other, and both matrices are tall---the regime wherein our extremely fast (but less accurate) encoding function is beneficial. Our method also loses utility when the larger matrix is known ahead of time; this assumption is common in similarity search, and eliminates the need for a fast encoding function entirely. Our approximate integer summation and fused table lookups would likely be useful independent of any of these assumptions, but demonstrating this is future work.
% It would be interesting to evaluate \textsc{MaddnessHash}\text{ } on its own as a tool for approximate nearest neighbors problems.

We also have several theoretical findings, taking the form of guarantees regarding the errors introduced by our method and its constituent subroutines. While we do obtain an overall generalization guarantee, this guarantee is not tight. In particular, it should grow looser with the large matrix's Frobenius norm and tighter as its singular values become more concentrated; at present, however, it simply grows looser as the largest singular value grows. The missing step is a guarantee that our encoding function will yield lower quantization errors when the singular values are more concentrated, which is its behavior in practice.

We have not demonstrated results using GPUs or other accelerators. While such accelerators are a small minority of hardware, they are often used in machine learning. Our method is not inherently tied to CPUs, but the differing performance characteristics of accelerators mean that adapting our method to them would require both algorithmic and implementation work, with the details depending on the device. We also have not evaluated a multi-CPU-threaded extension of our algorithm, though this is because our method is intended to serve as the low-level, compute-bound, block matrix product routine called by individual threads. % and 2) our method would parallelize in the same way as any other.

Finally, we have not demonstrated results using convolutional layers in neural networks, or results accelerating full networks. The weight reuse in convolutional layers presents many opportunities for algorithmic optimizations, and we hope to exploit them using a specialized extension of our method in future work. Accelerating overall networks will require two significant undertakings: first, the engineering work of building and integrating custom operators, data layouts, etc., into existing frameworks and networks; and second, the research necessary to determine when, how, and to what extent to include approximate kernels inspired by our approach. A particular difficulty with the latter is that our hash function is not differentiable.

We believe that accelerating full networks with our ideas is a promising direction, particularly for inference. This is especially true at the hardware level---our method requires only multi\textit{plexers}, not multi\textit{pliers}, and can therefore be implemented easily and with far less power than current matrix product logic. Moreover, our encoded representation and lookup tables have contiguous and uniformly-sized elements, making our approximate GEMM inner loops nearly identical to their dense counterparts--i.e., there is no need for complex access patterns or sparsity handling.

In summary, we introduced \ours, an algorithm that achieves up to a $10\times$ better speed-quality tradeoff than existing methods for the well-studied problem of approximate matrix multiplication (AMM), as measured on a large, diverse, and challenging set of real-world matrices. Our approach is a significant departure from existing AMM work in that it relies on hashing and table lookups rather than multiply-add operations. Our results suggest that future methods similar to our own might hold promise for accelerating convolution, deep learning, and other workloads bottlenecked by linear transforms.

\bibliographystyle{icml2021}
% \bibliography{prune,architectures,misc,understandDnn,classic,datasets,compress,science,metapapers}

% TODO uncomment

\bibliography{architectures,datasets,misc,sprintz,extract,bolt,backprop-alternatives,distillation,fast-dnn-runtime-stuff,fast-optimize,nets-theory,small-fast-arch,sparseAndPrune,scalarQuantize,vectorQuantize,combo,hashing,shrink+prune,binarize,ternary,automl,zero-shot,few-shot,amm,ammMore,lsh,adversarial,libs,math,metapapers}

\begin{thebibliography}{75}
\providecommand{\natexlab}[1]{#1}
\providecommand{\url}[1]{\texttt{#1}}
\expandafter\ifx\csname urlstyle\endcsname\relax
  \providecommand{\doi}[1]{doi: #1}\else
  \providecommand{\doi}{doi: \begingroup \urlstyle{rm}\Url}\fi

\bibitem[Abadi et~al.(2016)Abadi, Barham, Chen, Chen, Davis, Dean, Devin,
  Ghemawat, et~al.]{tensorflow}
Abadi, M., Barham, P., Chen, J., Chen, Z., Davis, A., Dean, J., Devin, M.,
  Ghemawat, S., et~al.
\newblock Tensorflow: A system for large-scale machine learning.
\newblock In \emph{OSDI}, volume~16, pp.\  265--283, 2016.

\bibitem[Achlioptas(2001)]{rademacherJL}
Achlioptas, D.
\newblock Database-friendly random projections.
\newblock In \emph{Proceedings of the twentieth ACM SIGMOD-SIGACT-SIGART
  symposium on Principles of database systems}, pp.\  274--281, 2001.

\bibitem[Ailon \& Chazelle(2009)Ailon and Chazelle]{fastjl}
Ailon, N. and Chazelle, B.
\newblock {The Fast Johnson-Lindenstrauss Transform and Approximate Nearest
  Neighbors}.
\newblock \emph{SIAM Journal on Computing (SICOMP)}, 39\penalty0 (1):\penalty0
  302--322, 2009.
\newblock \doi{10.1137/060673096}.

\bibitem[Andr{\'e} et~al.(2017)Andr{\'e}, Kermarrec, and
  Le~Scouarnec]{quickAdc}
Andr{\'e}, F., Kermarrec, A.-M., and Le~Scouarnec, N.
\newblock Accelerated nearest neighbor search with quick adc.
\newblock In \emph{Proceedings of the 2017 ACM on International Conference on
  Multimedia Retrieval}, pp.\  159--166, 2017.

\bibitem[Andr{\'e} et~al.(2019)Andr{\'e}, Kermarrec, and
  Le~Scouarnec]{quickerAdc}
Andr{\'e}, F., Kermarrec, A.-M., and Le~Scouarnec, N.
\newblock Quicker adc: Unlocking the hidden potential of product quantization
  with simd.
\newblock \emph{IEEE transactions on pattern analysis and machine
  intelligence}, 2019.

\bibitem[Babenko \& Lempitsky(2014)Babenko and Lempitsky]{aq}
Babenko, A. and Lempitsky, V.
\newblock Additive quantization for extreme vector compression.
\newblock In \emph{Proceedings of the IEEE Conference on Computer Vision and
  Pattern Recognition}, pp.\  931--938, 2014.

\bibitem[Babenko \& Lempitsky(2015)Babenko and Lempitsky]{otq}
Babenko, A. and Lempitsky, V.
\newblock {Tree Quantization for Large-Scale Similarity Search and
  Classification}.
\newblock \emph{CVPR}, pp.\  1--9, 2015.
\newblock URL
  \url{papers3://publication/uuid/F4762974-BB97-4208-B035-508945A90EFC}.

\bibitem[Babenko et~al.(2016)Babenko, Arandjelovi{\'c}, and Lempitsky]{pairq}
Babenko, A., Arandjelovi{\'c}, R., and Lempitsky, V.
\newblock Pairwise quantization.
\newblock \emph{arXiv preprint arXiv:1606.01550}, 2016.

\bibitem[Bakhtiary et~al.(2015)Bakhtiary, Lapedriza, and Masip]{wtaSoftmax}
Bakhtiary, A.~H., Lapedriza, A., and Masip, D.
\newblock Speeding up neural networks for large scale classification using wta
  hashing.
\newblock \emph{arXiv preprint arXiv:1504.07488}, 2015.

\bibitem[Bartlett \& Mendelson(2002)Bartlett and Mendelson]{rademacherOrig}
Bartlett, P.~L. and Mendelson, S.
\newblock Rademacher and gaussian complexities: Risk bounds and structural
  results.
\newblock \emph{Journal of Machine Learning Research}, 3\penalty0
  (Nov):\penalty0 463--482, 2002.

\bibitem[Blalock \& Guttag(2017)Blalock and Guttag]{bolt}
Blalock, D.~W. and Guttag, J.~V.
\newblock Bolt: Accelerated data mining with fast vector compression.
\newblock In \emph{Proceedings of the 23rd ACM SIGKDD International Conference
  on Knowledge Discovery and Data Mining}, pp.\  727--735. ACM, 2017.

\bibitem[Blalock et~al.(2020)Blalock, Ortiz, Frankle, and Guttag]{blalock2020}
Blalock, D.~W., Ortiz, J. J.~G., Frankle, J., and Guttag, J.~V.
\newblock What is the state of neural network pruning?
\newblock In Dhillon, I.~S., Papailiopoulos, D.~S., and Sze, V. (eds.),
  \emph{Proceedings of Machine Learning and Systems 2020, MLSys 2020, Austin,
  TX, USA, March 2-4, 2020}. mlsys.org, 2020.
\newblock URL \url{https://proceedings.mlsys.org/book/296.pdf}.

\bibitem[Camacho et~al.(2020)Camacho, Smilde, Saccenti, and
  Westerhuis]{spcaSurvey2}
Camacho, J., Smilde, A., Saccenti, E., and Westerhuis, J.
\newblock All sparse pca models are wrong, but some are useful. part i:
  Computation of scores, residuals and explained variance.
\newblock \emph{Chemometrics and Intelligent Laboratory Systems}, 196:\penalty0
  103907, 2020.

\bibitem[Chen et~al.(2019)Chen, Medini, and Shrivastava]{slide}
Chen, B., Medini, T., and Shrivastava, A.
\newblock Slide: In defense of smart algorithms over hardware acceleration for
  large-scale deep learning systems.
\newblock \emph{arXiv preprint arXiv:1903.03129}, 2019.

\bibitem[Chen et~al.(2015)Chen, Wilson, Tyree, Weinberger, and Chen]{hashnet}
Chen, W., Wilson, J.~T., Tyree, S., Weinberger, K.~Q., and Chen, Y.
\newblock Compressing neural networks with the hashing trick.
\newblock In \emph{ICML}, pp.\  2285--2294, 2015.

\bibitem[Chen et~al.(2016)Chen, Emer, and Sze]{eyeriss}
Chen, Y.-H., Emer, J., and Sze, V.
\newblock Eyeriss: A spatial architecture for energy-efficient dataflow for
  convolutional neural networks.
\newblock \emph{ACM SIGARCH Computer Architecture News}, 44\penalty0
  (3):\penalty0 367--379, 2016.

\bibitem[Dasgupta et~al.(2010)Dasgupta, Kumar, and Sarl{\'o}s]{hashjl}
Dasgupta, A., Kumar, R., and Sarl{\'o}s, T.
\newblock A sparse johnson: Lindenstrauss transform.
\newblock In \emph{Proceedings of the forty-second ACM symposium on Theory of
  computing}, pp.\  341--350, 2010.

\bibitem[Dau et~al.(2018)Dau, Keogh, Kamgar, Yeh, Zhu, Gharghabi,
  Ratanamahatana, Yanping, Hu, Begum, Bagnall, Mueen, Batista, and
  Hexagon-ML]{UCRArchive2018}
Dau, H.~A., Keogh, E., Kamgar, K., Yeh, C.-C.~M., Zhu, Y., Gharghabi, S.,
  Ratanamahatana, C.~A., Yanping, Hu, B., Begum, N., Bagnall, A., Mueen, A.,
  Batista, G., and Hexagon-ML.
\newblock The ucr time series classification archive, October 2018.
\newblock \url{https://www.cs.ucr.edu/~eamonn/time_series_data_2018/}.

\bibitem[Dean et~al.(2013)Dean, Ruzon, Segal, Shlens, Vijayanarasimhan, and
  Yagnik]{googleWtaCvpr}
Dean, T., Ruzon, M.~A., Segal, M., Shlens, J., Vijayanarasimhan, S., and
  Yagnik, J.
\newblock Fast, accurate detection of 100,000 object classes on a single
  machine.
\newblock In \emph{Proceedings of the IEEE Conference on Computer Vision and
  Pattern Recognition}, pp.\  1814--1821, 2013.

\bibitem[Desai et~al.(2016)Desai, Ghashami, and Phillips]{isvd}
Desai, A., Ghashami, M., and Phillips, J.~M.
\newblock Improved practical matrix sketching with guarantees.
\newblock \emph{IEEE Transactions on Knowledge and Data Engineering},
  28\penalty0 (7):\penalty0 1678--1690, 2016.

\bibitem[Drineas et~al.(2006{\natexlab{a}})Drineas, Kannan, and
  Mahoney]{drineas_fast_2006}
Drineas, P., Kannan, R., and Mahoney, M.~W.
\newblock Fast {Monte} {Carlo} {Algorithms} for {Matrices} {I}: {Approximating}
  {Matrix} {Multiplication}.
\newblock \emph{SIAM Journal on Computing}, 36\penalty0 (1):\penalty0 132--157,
  January 2006{\natexlab{a}}.
\newblock ISSN 0097-5397, 1095-7111.
\newblock \doi{10.1137/S0097539704442684}.
\newblock URL \url{http://epubs.siam.org/doi/10.1137/S0097539704442684}.

\bibitem[Drineas et~al.(2006{\natexlab{b}})Drineas, Kannan, and
  Mahoney]{drineas_fast_2006-1}
Drineas, P., Kannan, R., and Mahoney, M.~W.
\newblock Fast {Monte} {Carlo} {Algorithms} for {Matrices} {II}: {Computing} a
  {Low}-{Rank} {Approximation} to a {Matrix}.
\newblock \emph{SIAM Journal on Computing}, 36\penalty0 (1):\penalty0 158--183,
  January 2006{\natexlab{b}}.
\newblock ISSN 0097-5397, 1095-7111.
\newblock \doi{10.1137/S0097539704442696}.
\newblock URL \url{http://epubs.siam.org/doi/10.1137/S0097539704442696}.

\bibitem[Drineas et~al.(2006{\natexlab{c}})Drineas, Kannan, and
  Mahoney]{drineas_fast_2006-2}
Drineas, P., Kannan, R., and Mahoney, M.~W.
\newblock Fast {Monte} {Carlo} {Algorithms} for {Matrices} {III}: {Computing} a
  {Compressed} {Approximate} {Matrix} {Decomposition}.
\newblock \emph{SIAM Journal on Computing}, 36\penalty0 (1):\penalty0 184--206,
  January 2006{\natexlab{c}}.
\newblock ISSN 0097-5397, 1095-7111.
\newblock \doi{10.1137/S0097539704442702}.
\newblock URL \url{http://epubs.siam.org/doi/10.1137/S0097539704442702}.

\bibitem[Dutta et~al.(2016)Dutta, Cadambe, and Grover]{shortDot}
Dutta, S., Cadambe, V., and Grover, P.
\newblock Short-dot: Computing large linear transforms distributedly using
  coded short dot products.
\newblock In \emph{Advances In Neural Information Processing Systems}, pp.\
  2100--2108, 2016.

\bibitem[Eckart \& Young(1936)Eckart and Young]{eckartYoungMirskyThm}
Eckart, C. and Young, G.
\newblock The approximation of one matrix by another of lower rank.
\newblock \emph{Psychometrika}, 1\penalty0 (3):\penalty0 211--218, 1936.

\bibitem[Fei-Fei et~al.(2004)Fei-Fei, Fergus, and Perona]{caltech}
Fei-Fei, L., Fergus, R., and Perona, P.
\newblock Learning generative visual models from few training examples: An
  incremental bayesian approach tested on 101 object categories.
\newblock In \emph{2004 conference on computer vision and pattern recognition
  workshop}, pp.\  178--178. IEEE, 2004.

\bibitem[Francis \& Raimond(2018{\natexlab{a}})Francis and
  Raimond]{francis_improvement_2018}
Francis, D.~P. and Raimond, K.
\newblock An improvement of the parameterized frequent directions algorithm.
\newblock \emph{Data Mining and Knowledge Discovery}, 32\penalty0 (2):\penalty0
  453--482, March 2018{\natexlab{a}}.
\newblock ISSN 1384-5810, 1573-756X.
\newblock \doi{10.1007/s10618-017-0542-x}.
\newblock URL \url{http://link.springer.com/10.1007/s10618-017-0542-x}.

\bibitem[Francis \& Raimond(2018{\natexlab{b}})Francis and
  Raimond]{francis_practical_2018}
Francis, D.~P. and Raimond, K.
\newblock A practical streaming approximate matrix multiplication algorithm.
\newblock \emph{Journal of King Saud University - Computer and Information
  Sciences}, September 2018{\natexlab{b}}.
\newblock ISSN 13191578.
\newblock \doi{10.1016/j.jksuci.2018.09.010}.
\newblock URL
  \url{https://linkinghub.elsevier.com/retrieve/pii/S1319157818306396}.

\bibitem[Ge et~al.(2014)Ge, He, Ke, and Sun]{opq}
Ge, T., He, K., Ke, Q., and Sun, J.
\newblock Optimized product quantization.
\newblock \emph{IEEE transactions on pattern analysis and machine
  intelligence}, 36\penalty0 (4):\penalty0 744--755, 2014.

\bibitem[Geifman(2018)]{cifarVgg}
Geifman, Y.
\newblock cifar-vgg, 3 2018.
\newblock \url{https://github.com/geifmany/cifar-vgg}.

\bibitem[Ghashami et~al.(2016)Ghashami, Liberty, Phillips, and
  Woodruff]{ghashami_frequent_2016}
Ghashami, M., Liberty, E., Phillips, J.~M., and Woodruff, D.~P.
\newblock Frequent {Directions}: {Simple} and {Deterministic} {Matrix}
  {Sketching}.
\newblock \emph{SIAM Journal on Computing}, 45\penalty0 (5):\penalty0
  1762--1792, January 2016.
\newblock ISSN 0097-5397, 1095-7111.
\newblock \doi{10.1137/15M1009718}.
\newblock URL \url{http://epubs.siam.org/doi/10.1137/15M1009718}.

\bibitem[Guennebaud et~al.(2010)Guennebaud, Jacob, et~al.]{eigen}
Guennebaud, G., Jacob, B., et~al.
\newblock Eigen v3.
\newblock http://eigen.tuxfamily.org, 2010.

\bibitem[Gupta et~al.(2017)Gupta, Suggala, Goyal, Simhadri, Paranjape, Kumar,
  Goyal, Udupa, Varma, and Jain]{protonn}
Gupta, C., Suggala, A.~S., Goyal, A., Simhadri, H.~V., Paranjape, B., Kumar,
  A., Goyal, S., Udupa, R., Varma, M., and Jain, P.
\newblock Protonn: Compressed and accurate knn for resource-scarce devices.
\newblock In \emph{Proceedings of the 34th International Conference on Machine
  Learning-Volume 70}, pp.\  1331--1340. JMLR. org, 2017.

\bibitem[Han et~al.(2016)Han, Liu, Mao, Pu, Pedram, Horowitz, and Dally]{eie}
Han, S., Liu, X., Mao, H., Pu, J., Pedram, A., Horowitz, M.~A., and Dally,
  W.~J.
\newblock Eie: efficient inference engine on compressed deep neural network.
\newblock In \emph{Proceedings of the 43rd International Symposium on Computer
  Architecture}, pp.\  243--254. IEEE Press, 2016.

\bibitem[He et~al.(2016{\natexlab{a}})He, Zhang, Ren, and Sun]{resNet}
He, K., Zhang, X., Ren, S., and Sun, J.
\newblock Deep residual learning for image recognition.
\newblock In \emph{Proceedings of the IEEE conference on computer vision and
  pattern recognition}, pp.\  770--778, 2016{\natexlab{a}}.

\bibitem[He et~al.(2016{\natexlab{b}})He, Zhang, Ren, and Sun]{resnet2}
He, K., Zhang, X., Ren, S., and Sun, J.
\newblock Identity mappings in deep residual networks.
\newblock In \emph{European conference on computer vision}, pp.\  630--645.
  Springer, 2016{\natexlab{b}}.

\bibitem[Huang et~al.(2017)Huang, Liu, Van Der~Maaten, and
  Weinberger]{densenet}
Huang, G., Liu, Z., Van Der~Maaten, L., and Weinberger, K.~Q.
\newblock Densely connected convolutional networks.
\newblock In \emph{Proceedings of the IEEE conference on computer vision and
  pattern recognition}, pp.\  4700--4708, 2017.

\bibitem[Huang(2019)]{huang_near_2019}
Huang, Z.
\newblock Near {Optimal} {Frequent} {Directions} for {Sketching} {Dense} and
  {Sparse} {Matrices}.
\newblock \emph{Journal of Machine Learning Research}, 20\penalty0
  (1):\penalty0 23, February 2019.

\bibitem[Indyk \& Motwani(1998)Indyk and Motwani]{lshOrig}
Indyk, P. and Motwani, R.
\newblock Approximate nearest neighbors: towards removing the curse of
  dimensionality.
\newblock In \emph{Proceedings of the thirtieth annual ACM symposium on Theory
  of computing}, pp.\  604--613, 1998.

\bibitem[Irony et~al.(2004)Irony, Toledo, and
  Tiskin]{matmulCommunicationBounds}
Irony, D., Toledo, S., and Tiskin, A.
\newblock Communication lower bounds for distributed-memory matrix
  multiplication.
\newblock \emph{Journal of Parallel and Distributed Computing}, 64\penalty0
  (9):\penalty0 1017--1026, 2004.

\bibitem[Jegou et~al.(2011)Jegou, Douze, and Schmid]{pq}
Jegou, H., Douze, M., and Schmid, C.
\newblock Product quantization for nearest neighbor search.
\newblock \emph{IEEE transactions on pattern analysis and machine
  intelligence}, 33\penalty0 (1):\penalty0 117--128, 2011.

\bibitem[Ji et~al.(2012)Ji, Li, Yan, Zhang, and Tian]{superbitLSH}
Ji, J., Li, J., Yan, S., Zhang, B., and Tian, Q.
\newblock Super-bit locality-sensitive hashing.
\newblock In \emph{Advances in Neural Information Processing Systems}, pp.\
  108--116, 2012.

\bibitem[Jouppi et~al.(2017)Jouppi, Young, Patil, Patterson, Agrawal, Bajwa,
  Bates, Bhatia, Boden, Borchers, et~al.]{tpu}
Jouppi, N.~P., Young, C., Patil, N., Patterson, D., Agrawal, G., Bajwa, R.,
  Bates, S., Bhatia, S., Boden, N., Borchers, A., et~al.
\newblock In-datacenter performance analysis of a tensor processing unit.
\newblock In \emph{Proceedings of the 44th Annual International Symposium on
  Computer Architecture}, pp.\  1--12. ACM, 2017.

\bibitem[Kakade et~al.(2009)Kakade, Sridharan, and Tewari]{kakadeLinear}
Kakade, S.~M., Sridharan, K., and Tewari, A.
\newblock On the complexity of linear prediction: Risk bounds, margin bounds,
  and regularization.
\newblock In \emph{Advances in neural information processing systems}, pp.\
  793--800, 2009.

\bibitem[Khudia et~al.(2018)Khudia, Basu, and Deng]{fbgemm}
Khudia, D., Basu, P., and Deng, S.
\newblock Open-sourcing fbgemm for state-of-the-art server-side inference,
  2018.

\bibitem[Krizhevsky et~al.(2009)Krizhevsky, Hinton, et~al.]{cifarDsets}
Krizhevsky, A., Hinton, G., et~al.
\newblock Learning multiple layers of features from tiny images.
\newblock Technical report, Citeseer, 2009.

\bibitem[Kusner et~al.(2014)Kusner, Tyree, Weinberger, and Agrawal]{snc}
Kusner, M., Tyree, S., Weinberger, K., and Agrawal, K.
\newblock Stochastic neighbor compression.
\newblock In \emph{International Conference on Machine Learning}, pp.\
  622--630, 2014.

\bibitem[Kyrillidis et~al.(2014)Kyrillidis, Vlachos, and
  Zouzias]{kyrillidis_approximate_2014}
Kyrillidis, A., Vlachos, M., and Zouzias, A.
\newblock Approximate {Matrix} {Multiplication} with {Application} to {Linear}
  {Embeddings}.
\newblock \emph{arXiv:1403.7683 [cs, math, stat]}, March 2014.
\newblock URL \url{http://arxiv.org/abs/1403.7683}.
\newblock arXiv: 1403.7683.

\bibitem[Liberty(2012)]{liberty_simple_2012}
Liberty, E.
\newblock Simple and {Deterministic} {Matrix} {Sketching}.
\newblock \emph{arXiv:1206.0594 [cs]}, June 2012.
\newblock URL \url{http://arxiv.org/abs/1206.0594}.
\newblock arXiv: 1206.0594.

\bibitem[Liu et~al.(2016)Liu, Shao, and Lu]{grvq}
Liu, S., Shao, J., and Lu, H.
\newblock {Generalized Residual Vector Quantization for Large Scale Data}.
\newblock \emph{Proceedings - IEEE International Conference on Multimedia and
  Expo}, 2016-Augus, 2016.
\newblock ISSN 1945788X.
\newblock \doi{10.1109/ICME.2016.7552944}.

\bibitem[Luo et~al.(2019)Luo, Chen, Zhang, Li, and Zhang]{luo_robust_2019}
Luo, L., Chen, C., Zhang, Z., Li, W.-J., and Zhang, T.
\newblock Robust {Frequent} {Directions} with {Application} in {Online}
  {Learning}.
\newblock \emph{Journal of Machine Learning Research}, 20\penalty0
  (1):\penalty0 41, February 2019.

\bibitem[Mairal et~al.(2009)Mairal, Bach, Ponce, and Sapiro]{sparsePCA}
Mairal, J., Bach, F., Ponce, J., and Sapiro, G.
\newblock Online dictionary learning for sparse coding.
\newblock In \emph{Proceedings of the 26th annual international conference on
  machine learning}, pp.\  689--696, 2009.

\bibitem[Manne \& Pal(2014)Manne and Pal]{manne_fast_2014}
Manne, S. and Pal, M.
\newblock Fast {Approximate} {Matrix} {Multiplication} by {Solving} {Linear}
  {Systems}.
\newblock \emph{arXiv:1408.4230 [cs]}, August 2014.
\newblock URL \url{http://arxiv.org/abs/1408.4230}.
\newblock arXiv: 1408.4230.

\bibitem[Martinez et~al.(2014)Martinez, Hoos, and Little]{stackedQuantizers}
Martinez, J., Hoos, H.~H., and Little, J.~J.
\newblock Stacked quantizers for compositional vector compression.
\newblock \emph{arXiv preprint arXiv:1411.2173}, 2014.

\bibitem[Martinez et~al.(2016)Martinez, Clement, Hoos, and Little]{lsq}
Martinez, J., Clement, J., Hoos, H.~H., and Little, J.~J.
\newblock Revisiting additive quantization.
\newblock In \emph{European Conference on Computer Vision}, pp.\  137--153.
  Springer, 2016.

\bibitem[Mroueh et~al.(2016)Mroueh, Marcheret, and
  Goel]{mroueh_co-occuring_2016}
Mroueh, Y., Marcheret, E., and Goel, V.
\newblock Co-{Occuring} {Directions} {Sketching} for {Approximate} {Matrix}
  {Multiply}.
\newblock \emph{arXiv:1610.07686 [cs]}, October 2016.
\newblock URL \url{http://arxiv.org/abs/1610.07686}.
\newblock arXiv: 1610.07686.

\bibitem[Nelson \& Nguy{\^e}n(2013)Nelson and Nguy{\^e}n]{osnap}
Nelson, J. and Nguy{\^e}n, H.~L.
\newblock Osnap: Faster numerical linear algebra algorithms via sparser
  subspace embeddings.
\newblock In \emph{2013 ieee 54th annual symposium on foundations of computer
  science}, pp.\  117--126. IEEE, 2013.

\bibitem[Pagh(2013)]{pagh_compressed_2013}
Pagh, R.
\newblock Compressed matrix multiplication.
\newblock \emph{ACM Transactions on Computation Theory}, 5\penalty0
  (3):\penalty0 1--17, August 2013.
\newblock ISSN 19423454.
\newblock \doi{10.1145/2493252.2493254}.
\newblock URL \url{http://dl.acm.org/citation.cfm?doid=2493252.2493254}.

\bibitem[Parashar et~al.(2017)Parashar, Rhu, Mukkara, Puglielli, Venkatesan,
  Khailany, Emer, Keckler, and Dally]{scnn}
Parashar, A., Rhu, M., Mukkara, A., Puglielli, A., Venkatesan, R., Khailany,
  B., Emer, J., Keckler, S.~W., and Dally, W.~J.
\newblock Scnn: An accelerator for compressed-sparse convolutional neural
  networks.
\newblock \emph{ACM SIGARCH Computer Architecture News}, 45\penalty0
  (2):\penalty0 27--40, 2017.

\bibitem[Paszke et~al.(2017)Paszke, Gross, Chintala, Chanan, Yang, DeVito, Lin,
  Desmaison, Antiga, and Lerer]{pytorch}
Paszke, A., Gross, S., Chintala, S., Chanan, G., Yang, E., DeVito, Z., Lin, Z.,
  Desmaison, A., Antiga, L., and Lerer, A.
\newblock Automatic differentiation in pytorch.
\newblock 2017.

\bibitem[Sarlos(2006)]{sarlos_improved_2006}
Sarlos, T.
\newblock Improved {Approximation} {Algorithms} for {Large} {Matrices} via
  {Random} {Projections}.
\newblock In \emph{2006 47th {Annual} {IEEE} {Symposium} on {Foundations} of
  {Computer} {Science} ({FOCS}'06)}, pp.\  143--152, Berkeley, CA, October
  2006. IEEE.
\newblock ISBN 978-0-7695-2720-8.
\newblock \doi{10.1109/FOCS.2006.37}.
\newblock URL \url{https://ieeexplore.ieee.org/document/4031351/}.

\bibitem[Spring \& Shrivastava(2017)Spring and Shrivastava]{springScalable}
Spring, R. and Shrivastava, A.
\newblock Scalable and sustainable deep learning via randomized hashing.
\newblock In \emph{Proceedings of the 23rd ACM SIGKDD International Conference
  on Knowledge Discovery and Data Mining}, pp.\  445--454, 2017.

\bibitem[Teng \& Chu(2019)Teng and Chu]{teng_fast_2019}
Teng, D. and Chu, D.
\newblock A {Fast} {Frequent} {Directions} {Algorithm} for {Low} {Rank}
  {Approximation}.
\newblock \emph{IEEE Transactions on Pattern Analysis and Machine
  Intelligence}, 41\penalty0 (6):\penalty0 1279--1293, June 2019.
\newblock ISSN 0162-8828, 2160-9292, 1939-3539.
\newblock \doi{10.1109/TPAMI.2018.2839198}.
\newblock URL \url{https://ieeexplore.ieee.org/document/8362693/}.

\bibitem[Wang et~al.(2014{\natexlab{a}})Wang, Shen, Song, and
  Ji]{hashingSimilaritySurvey}
Wang, J., Shen, H.~T., Song, J., and Ji, J.
\newblock Hashing for similarity search: A survey.
\newblock \emph{arXiv preprint arXiv:1408.2927}, 2014{\natexlab{a}}.

\bibitem[Wang et~al.(2014{\natexlab{b}})Wang, Shen, Yan, Yu, Li, and
  Wang]{optimizedDists}
Wang, J., Shen, H.~T., Yan, S., Yu, N., Li, S., and Wang, J.
\newblock Optimized distances for binary code ranking.
\newblock In \emph{Proceedings of the 22nd ACM international conference on
  Multimedia}, pp.\  517--526, 2014{\natexlab{b}}.

\bibitem[Wang et~al.(2016{\natexlab{a}})Wang, Liu, Kumar, and
  Chang]{learningToHashSurvey}
Wang, J., Liu, W., Kumar, S., and Chang, S.-F.
\newblock Learning to hash for indexing big data—a survey.
\newblock \emph{Proceedings of the IEEE}, 104\penalty0 (1):\penalty0 34--57,
  2016{\natexlab{a}}.

\bibitem[Wang et~al.(2016{\natexlab{b}})Wang, Chen, Chen, Rai, and Carin]{dsnc}
Wang, W., Chen, C., Chen, W., Rai, P., and Carin, L.
\newblock Deep metric learning with data summarization.
\newblock In \emph{Joint European Conference on Machine Learning and Knowledge
  Discovery in Databases}, pp.\  777--794. Springer, 2016{\natexlab{b}}.

\bibitem[Wu et~al.(2019)Wu, Brooks, Chen, Chen, Choudhury, Dukhan, Hazelwood,
  Isaac, Jia, Jia, et~al.]{fbAtEdge}
Wu, C.-J., Brooks, D., Chen, K., Chen, D., Choudhury, S., Dukhan, M.,
  Hazelwood, K., Isaac, E., Jia, Y., Jia, B., et~al.
\newblock Machine learning at facebook: Understanding inference at the edge.
\newblock In \emph{2019 IEEE International Symposium on High Performance
  Computer Architecture (HPCA)}, pp.\  331--344. IEEE, 2019.

\bibitem[Ye et~al.(2016)Ye, Luo, and Zhang]{ye_frequent_2016}
Ye, Q., Luo, L., and Zhang, Z.
\newblock Frequent {Direction} {Algorithms} for {Approximate} {Matrix}
  {Multiplication} with {Applications} in {CCA}.
\newblock In \emph{{IJCAI}}, pp.\ ~7, 2016.

\bibitem[Yu et~al.(2017)Yu, Maddah-Ali, and Avestimehr]{distributedCoded}
Yu, Q., Maddah-Ali, M., and Avestimehr, S.
\newblock Polynomial codes: an optimal design for high-dimensional coded matrix
  multiplication.
\newblock In \emph{Advances in Neural Information Processing Systems}, pp.\
  4403--4413, 2017.

\bibitem[Yu et~al.(2020)Yu, Ali, and Avestimehr]{entangledPolynomial}
Yu, Q., Ali, M., and Avestimehr, A.~S.
\newblock Straggler mitigation in distributed matrix multiplication:
  Fundamental limits and optimal coding.
\newblock \emph{IEEE Transactions on Information Theory}, 2020.

\bibitem[Zhang et~al.(2014)Zhang, Du, and Wang]{cq}
Zhang, T., Du, C., and Wang, J.
\newblock {Composite Quantization for Approximate Nearest Neighbor Search}.
\newblock \emph{Proceedings of the 31st International Conference on Machine
  Learning (ICML-14)}, 32:\penalty0 838--846, 2014.

\bibitem[Zhang et~al.(2015)Zhang, Qi, Tang, and Wang]{scq}
Zhang, T., Qi, G.-J., Tang, J., and Wang, J.
\newblock Sparse composite quantization.
\newblock In \emph{Proceedings of the IEEE Conference on Computer Vision and
  Pattern Recognition}, pp.\  4548--4556, 2015.

\bibitem[Zhong et~al.(2017)Zhong, Guo, Kumar, Yan, Simcha, and Dhillon]{bnc}
Zhong, K., Guo, R., Kumar, S., Yan, B., Simcha, D., and Dhillon, I.
\newblock Fast classification with binary prototypes.
\newblock In \emph{Artificial Intelligence and Statistics}, pp.\  1255--1263,
  2017.

\bibitem[Zou \& Xue(2018)Zou and Xue]{spcaSurvey1}
Zou, H. and Xue, L.
\newblock A selective overview of sparse principal component analysis.
\newblock \emph{Proceedings of the IEEE}, 106\penalty0 (8):\penalty0
  1311--1320, 2018.

\end{thebibliography}
\clearpage
\newpage  % so we can cut the pdf
\appendix

% ================================================================
\section{Quantizing Lookup Tables} \label{sec:lutQuantize}
% ================================================================

% Existing work has shown that setting $K = 16$ and quantizing the lookup tables for a given column of $\B$ to 8 bits can offer enormous speedups compared to larger $K$ and/or floating-point tables \cite{bolt, quickAdc, quickerAdc}. This is because 16 1-byte entries can be stored in a SIMD register, allowing 16 or more table lookups to be performed in parallel in a single instruction.

Since the lookup table entries naturally occupy more than 8 bits even for 8-bit data (since products of 8-bit values require 16 bits), some means of quantizing these entries is necessary to enable vectorization. Unfortunately, existing quantization methods are not applicable to our problem setting. The scheme of \citet{bolt} requires knowledge of $\B$ at training time, while the scheme of \citet{quickAdc} and \citet{quickerAdc} is only applicable for nearest-neighbor search. We instead use the following approach, where $\mat{T} \in \R^{M \times C \times K}$ is the tensor of lookup tables for all $M$ columns of $\B$, $\mat{T}^q$ is the quantized version of $\mat{T}$, $\vec{\delta} \in \R^C$ is a vector of table-specific offsets, and $\alpha^{-1}$ is an overall scale factor:
\begin{align}
    \vec{\delta}_c &\triangleq \min_{m,k} \text{ } \mat{T}_{m,c,k} \\
    \alpha^{-1} \triangleq 2^l \cs l &= \max_c \floor{ \log_2 \left( \frac{255}{
            \max_{m,k} (\mat{T}_{m,c,k} - \delta_{c})
        } \right)} \\
    \mat{T}^q_{m,c,k} &\triangleq \alpha^{-1}(\mat{T}_{m,c,k} - \delta_{c}).
\end{align}
This is similar to equations~\ref{eq:offset} and \ref{eq:scale}, but with the scale factor pooled across all codebooks instead of unique to each input column. The $\alpha$ used here is the same as that in equation~\ref{eq:objective}, and the matrix $\beta$ in equation~\ref{eq:objective} has entries equal to $\sum_c \vec{\delta}_c$ (plus the debiasing constant from our averaging-based aggregation).

% ================================================================
\section{Quantization and \oursHash} \label{sec:hashQuantize}
% ================================================================

The use of at most 16 leaves is so that the resulting codes use 4 bits. This allows the use of these same shuffle instructions to accelerate the table lookups as in \citet{bolt}.

The only subtlety in vectorizing our hash function is that one must execute line 4 using shuffle instructions such as \texttt{vpshufb} on x86, \texttt{vtbl} on ARM, or \texttt{vtbl} on PowerPC. In order to do this, the split values and scalars $x_{j^t}$ must be 8-bit integers. We quantize them by learning for each split index $j$ a pair of scalars $(\gamma_j, \delta_j)$, where
\begin{align}
    \delta_j &\triangleq \min_i \v^j_i \label{eq:offset} \\
    % \gamma_j \triangleq 2^l, l = \left \lfloor \text{log2} \left( \frac{255}{\max_i \v^j_i - \delta_j} \right) \right \rfloor
    \gamma_j &\triangleq 2^l, l = \floor{ \log_2 \left( \frac{255}{\max_i \v^j_i - \delta_j} \right) \label{eq:scale} }
\end{align}
% and
% \begin{align}
% %     &\min_i \v^j_i - \delta_j = 0 \\
% %     &\max_i \gamma_j(\v^j_i - \delta_j) \le 255
% \end{align}
% and $\gamma_j$ is a power of 2.
This restriction of $\gamma_j$ to powers of two allows one to quantize $x_{j^t}$ values with only shifts instead of multiplies. The $\vec{v}$ values can be quantized at the end of the training phase, while the $x_{j^t}$ values must be quantized within Algorithm~\ref{algo:ourEnc} before line 5.

% ================================================================
\vfill\break  % columnbreak
\section{Subroutines for Training \oursHash} \label{sec:optimalSplitVal}
% ================================================================

The \texttt{optimal\_split\_threshold} algorithm (Algorithm~\ref{algo:optimalSplitVal}) finds the best threshold at which to split a bucket within a given dimension. To do this, it uses the \texttt{cumulative\_sse} function (Algorithm \ref{algo:cumSSE}) to help evaluate the loss associated with the resulting child buckets.

These algorithms exploit the fact that the sum of squared errors can be computed using only the sum of values and sum of squared values, both of which can be updated in $O(1)$ time when a vector is moved from one side of the split to the other.

% ------------------------------------------------ optimalSplitVal

\begin{algorithm}[h]
\caption{Optimal Split Threshold Within a Bucket} \label{algo:optimalSplitVal}
\begin{algorithmic}[1]
    \STATE {\bfseries Input:} bucket $\mathcal{B}$, index $j$
    \STATE {$\mat{X} \leftarrow \texttt{as\_2d\_array}(\mathcal{B})$ }
    \STATE {$\mat{X}^{sort} = \texttt{sort\_rows\_based\_on\_col}(\mat{X} \text{, } j)$}
    \STATE {$\texttt{sses\_head} \leftarrow \texttt{cumulative\_sse}(\mat{X}^{sort}, \texttt{false}) $}
    \STATE {$\texttt{sses\_tail} \leftarrow \texttt{cumulative\_sse}(\mat{X}^{sort}, \texttt{true}) $}
    \STATE {$\texttt{losses} \leftarrow \texttt{sses\_head} $}
    \STATE {$\texttt{losses}_{1:N-1} \leftarrow \texttt{losses}_{1:N-1} + \texttt{sses\_tail}_{2:N} $}
    % \STATE {$ n^\ast \leftarrow \min_n \sum_{d=1}^D $}
    % \STATE {$ n^\ast \leftarrow \argmin_n \texttt{sses\_head}_n + \texttt{sses\_tail}_{n+1} $}
    \STATE {$ n^\ast \leftarrow \argmin_n \texttt{losses}_n $}
    \STATE{$ \textbf{return } (\mat{X}^{sort}_{n^\ast \text{, } j} + \mat{X}^{sort}_{n^\ast + 1 \text{, } j}) / 2 \text{, } \texttt{losses}_{n^\ast} $}
    % \STATE {$\texttt{sses\_tail} \leftarrow \texttt{reverse\_row\_order}(\texttt{cumsse\_cols}(\texttt{reverse\_row\_order}(\mat{X}_{sort})_) $}
    % \STATE {$\texttt{sses\_tail} \leftarrow \texttt{reverse\_row\_order}(\texttt{cumsse\_cols}(\texttt{reverse\_row\_order}(\mat{X}_{sort})_) $}

\end{algorithmic}
\end{algorithm}

% ------------------------------------------------ cumSSE

\begin{algorithm}[h]
% \caption{Cumulative SSE Within Columns} \label{algo:cumSSE}
\caption{Cumulative SSE} \label{algo:cumSSE}
\begin{algorithmic}[1]
    \STATE {\bfseries Input:} 2D array $\mat{X}$, boolean \texttt{reverse}
    \STATE {$N, D \leftarrow \texttt{shape}(\mat{X})$}
    \IF {\texttt{reverse}}
        \STATE{$ \forall_i \texttt{ swap}(\mat{X}_{i,d}, \mat{X}_{N-i+1,d}) $}
    \ENDIF

    % \STATE {$\texttt{out} \leftarrow \texttt{empty}(N \text{, } D)$}
    \STATE {$\texttt{out} \leftarrow \texttt{empty}(N)$}
    \STATE {$\texttt{cumX} \leftarrow \texttt{empty}(D)$}
    \STATE {$\texttt{cumX2} \leftarrow \texttt{empty}(D)$}

    \LINECOMMENT{Initialize first row of output and cumulative values}
    \STATE{$\texttt{out}_{1} \leftarrow 0 $}
    \FOR{$d \leftarrow 1 \textbf{ to } D $}
        \STATE{$\texttt{cumX}_d \leftarrow X_{1, d} $}
        \STATE{$\texttt{cumX2}_d \leftarrow (X_{1, d})^2 $}
        % \STATE{$\texttt{out}_{1, d} \leftarrow 0 $}
    \ENDFOR

    \LINECOMMENT{Compute remaining output rows}
    \FOR{$n \leftarrow 2 \textbf{ to } N $}
        \STATE{$\texttt{out}_{n} \leftarrow 0 $}
        \FOR{$d \leftarrow 1 \textbf{ to } D $}
            \STATE{$\texttt{cumX}_d \leftarrow \texttt{cumX}_d + X_{1, d} $}
            \STATE{$\texttt{cumX2}_d \leftarrow \texttt{cumX2}_d + (X_{1, d})^2 $}
            % \STATE{$\texttt{meanX} \leftarrow \texttt{cumX}_d / n $}
            % \STATE{$\texttt{out}_{n, d} \leftarrow \texttt{cumX2}_d - (\texttt{cumX}_d \times \texttt{cumX}_d / n)$}
            \STATE{$\texttt{out}_{n} \leftarrow \texttt{out}_{n} + \texttt{cumX2}_d - (\texttt{cumX}_d \times \texttt{cumX}_d / n)$}
        \ENDFOR
    \ENDFOR
    \STATE{\textbf{return } \texttt{out}}
\end{algorithmic}
\end{algorithm}

% ================================================================
% \vfill\break  % columnbreak
\clearpage
\section{Aggregation Using Pairwise Averages} \label{sec:aggregateAnalysis}
% ================================================================

Recall that we estimate sums of low-bitwidth integers by averaging pairs of values, then pairs of pairs, and so on. One could reduce all $C$ values this way, but we find that one obtains a better speed-accuracy tradeoff by computing the average of blocks of $U$ values and then upcasting to obtain exact sums of these averages. Multiplying this sum of averages by $U$ and adding in a bias correction term gives one the overall estimate of the sum. One could tune $U$ for a particular problem and hardware, but we simply set $U = 16$ in all our experiments. Having a larger $U$ imposes less overhead because upcasting happens less often, but there are sharp diminishing returns to this; once upcasting is rare, doing it even less often is of little help thanks to Amdahl's law.

Because of our assumption that we are operating on matrices, rather than a matrix and a vector, we can also improve on the aggregation of existing methods \cite{bolt, quickAdc, quickerAdc} by fusing the aggregation of two or more output columns to hide read latency. Conceptually, this amounts to tiling the loop over output columns and alternating reads between the two corresponding tables within the innermost loop. This fusion does not change the output of the computation---only how efficiently it runs.

Having addressed these practical details, we may now proceed to the analysis of our estimator's bias.

\begin{definition}[Averaging Integer Sum Estimator]
Let $\x \in \{0, 1\}^C, C \text{ \% } U = 0, U = 2^p, p \ge 0$. The Averaging Integer Sum Estimator (AISE) $\hat{s}(\x)$ is defined as:
\begin{align}
    % s(\x) &\triangleq \sum_{k=1}^{C / U} s_U(\x_{((k-1)*U+1):((k-1)*U+1+U)}) \\
    \hat{s}(\x) &\triangleq \sum_{k=1}^{C / U} \hat{s}_U(\x_{i_k:j_k}) \\
    \hat{s}_U(\x) &\triangleq
        \begin{cases}
            x_1 & \x \in \R^1 \\
            % \lfloor \frac{1}{2}(x_1 + x_2 + 1) \rfloor & \x \in \R^2 \\
             % \lfloor \frac{1}{2}(s_U(\x_{:\text{len}(x)/2}) + s_U(\x_{\text{len}(x)/2:}) + 1) \rfloor & \text{otherwise}
             \lfloor \frac{1}{2}(\hat{s}_U(\x_{left}) + \hat{s}_U(\x_{right}) + 1) \rfloor & \text{otherwise}
       \end{cases}
\end{align}
where $i_k = (k-1) \cdot U+1, j_k = i_U + U$ and $\x_{left}$ and $\x_{right}$ denote vectors formed by taking the initial and final $D/2$ indices of a given $\x \in \R^D$.
\end{definition}

\begin{definition}[Pairwise Integer Sum and Sum Estimator]
For integers $a$ and $b$, define
\begin{align}
    s(a, b) &\triangleq a + b \\
    \hat{s}(a, b) &\triangleq 2 \mu(a, b)
\end{align}
where $\mu(a, b) \triangleq \lfloor \frac{1}{2}(a + b + 1) \rfloor$.
\end{definition}

\begin{lemma}[Bias when averaging one pair] \label{thm:avgBias}
Consider two scalars $a$ and $b$, with $a, b \iid$ Bernoulli(.5). Define $\eps(a, b) \triangleq \hat{s}(a, b) - s(a, b)$. Then
\[
    E[\eps(a, b)] = \frac{1}{2}
\]
\end{lemma}
\begin{proof} The proof follows immediately from considering the four equiprobable realizations of the pair $a, b$. In the cases $(0, 0)$ and $(1, 1)$, $2\mu(a, b) = s(a, b)$. In the cases $(0, 1)$ and $(1, 0)$, $2 \mu(a, b) = 2$, while $s(a, b) = 1$.
% : $(0, 0)$, $(0, 1)$, $(1, 0)$, $(1, 1)$
\end{proof}

% \begin{lemma}[Bias when averaging one pair] \label{thm:avgBias}
% Consider two scalars $a$ and $b$, $a, b \iid Bernoulli(.5)$. Then
% \[
%     E[\mu(x, y)] = \frac{1}{2}
% \]
% \end{lemma}
% \begin{proof} The proof follows immediately from considering the four equiprobable realizations of the pair $a, b$. In the cases $(0, 0)$ and $(1, 1)$, $\mu(a, b) = 0$. In the cases $(0, 1)$ and $(1, 0)$, $\mu(a, b) = 1$.
% % : $(0, 0)$, $(0, 1)$, $(1, 0)$, $(1, 1)$
% \end{proof}

\begin{lemma}[Variance of error when averaging one pair]
Consider two scalars $a$ and $b$, $a, b \iid$ Bernoulli(.5). Then
\[
    % E[\mu(a, b)^2] - E[\mu(x, y)]^2 = \frac{1}{4}
    E[\eps(a, b)^2] - E[\eps(x, y)]^2 = \frac{1}{4}
\]
\end{lemma}
\begin{proof} Using Lemma~\ref{thm:avgBias}, the above can be rewritten as:
\[
    % E[\mu(a, b)^2] = \frac{1}{2}
    E[\eps(a, b)^2] = \frac{1}{2}
\]
The proof then follows by again considering the four equiprobable cases as in Lemma~\ref{thm:avgBias}. In the cases $(0, 0)$ and $(1, 1)$, $\eps(a, b)^2 = 0$. In the cases $(0, 1)$ and $(1, 0)$, $(2 \hat{s}(a, b) - s(a, b))^2 = (2 - 1)^2 = 1$.
\end{proof}

% \begin{lemma}[Bias Recursive Step]
% % When applying the function $\mu$ recursively,
% \[
%     % E[\mu(\mu(a, b), \mu(c, d))] = 1
%     % E[\mu(\mu(a, b), \mu(c, d))] = 2 (E[\mu(a, b)] + E[\mu(c, d)])
%     % E[\eps(\hat{s}(a, b), \hat{s}(c, d))] = .5 + E[\eps(a, b)] + E[\eps(c, d)]
%     E[\eps(\mu(a, b), \mu{s}(c, d))] = .5 + E[\eps(a, b)] + E[\eps(c, d)]
% \]
% \end{lemma}
% \begin{proof}This can equivalently be formulated as
% \begin{align}
%     % E[\mu(\mu(a, b), \mu(c, d))] = 2 E[\mu(\mu(a, b), \mu(c, d))]
%     E[\mu(\mu(a, b), \mu(c, d))] = 2 E[\mu(a, b)] + E[\mu(c, d)]
% \end{align}
% \end{proof}

% \begin{lemma}[Bias and error variance of AISE within a subspace] \label{thm:avgSubspace}
\begin{lemma}[Bias of AISE within a subspace] \label{thm:avgSubspace}
% Let $\hat{s}(\x), \x \in \R^C, C \text{ \% } U = 0$ be the sum estimator described in section~\ref{sec:aggregate} and let $s(\x) \triangleq \sum_c x_c$.
% Let $s(\x)$ be defined as
Suppose that the scalar elements $x_i$ of $\vec{x}$ are drawn from independent Bernoulli(.5) distributions. Then
\begin{align}
    E[s_U(\x) - \hat{s}_U(\x)] &= U \log_2(U) / 4
    % Var[s_U(\x) - \hat{s}_U(\x)] &= U \log_2(U) / 8.
\end{align}
\end{lemma}
\begin{proof}
Observe that the computation graph can be cast as a balanced binary tree with $U$ leaves and each parent equal to the integer average of its children. Consider the bias introduced at each level $t$ of the tree, where $t=0$ corresponds to the leaves and $t = \log_2(U)$ corresponds to the root. The expected error $E[\xi(t, n)]$ introduced at a node $n$ in level $t > 0$ is given by:
\begin{align}
    E[\xi(t, n)] = \frac{1}{2} \cdot 2^{t - 1}
\end{align}
where the $\frac{1}{2}$ follows from Lemma~\ref{thm:avgBias} and the scale $2^{t - 1}$ is the number of leaf nodes to which the bias is effectively applied. E.g., adding one to the estimated average of four leaf nodes would increase the estimated sum by four. Since there are $U \cdot 2^{-t}$ nodes per level, this means that the total expected error introduced at level $t$ is $\frac{1}{2} \cdot 2^{t - 1} \cdot 2^{-t} = \frac{1}{4}$. Summing from $t = 1$ to $t = \log_2(U)$ completes the proof of the expected error. Note that $t=0$ is omitted since the leaf nodes are not the result of averaging operations and so introduce no error.

% Identical reasoning can be used to derive the variance of the error, since the variances from different nodes add.
\end{proof}

% \begin{theorem}[Bias and error variance of AISE]
\begin{theorem}[Bias of AISE] \label{thm:overallBias}
% Let $\hat{s}(\x), \x \in \R^C, C \text{ \% } U = 0$ be the sum estimator described in section~\ref{sec:aggregate} and let $s(\x) \triangleq \sum_c x_c$.
% Let $s(\x)$ be defined as
Suppose that the scalar elements $x_i$ of $\vec{x}$ are drawn from independent Bernoulli(.5) distributions. Then
\begin{align}
    E[s(\x) - \hat{s}(\x)] &= C \log_2(U) / 4
    % Var[s(\x) - \hat{s}(\x)] &= C \log_2(U) / 8.
\end{align}
\end{theorem}
\begin{proof}
This follows immediately from Lemma~\ref{thm:avgSubspace}, the fact that the overall sum is estimated within each of $C / U$ subspaces of size $U$, and the assumption that the errors in each subspace are independent.
\end{proof}

We also verified Theorem~\ref{thm:overallBias} numerically by summing large numbers of integers drawn uniformly from the interval $0,\ldots,255$.

Note that the assumption that the elements are independent is not especially strong in reality. This is because this section focuses on the effects on the least significant bits (which are the ones affected by each averaging operation), and the least significant bit does tend to be nearly uniformly random in a great deal of real-world data.

% ================================================================
\vfill\break  % columnbreak
\section{Additional Experimental Details} \label{sec:experimentDetails}
% ================================================================

% ------------------------------------------------
\subsection{Choice of Matrix Multiplication Tasks}
% ------------------------------------------------

Because nearly all existing work on approximate matrix multiplication either focuses on special cases that do not satisfy our problem definition \cite{quickerAdc, pq, opq} or synthetic matrices, there is not a clear set of benchmark matrix multiply tasks to use. We therefore propose a collection of tasks that we believe are both reproducible and representative of many real-world matrices. To the best of our knowledge, our experiments use over an order of magnitude more matrices than any previous study.

% ------------------------------------------------
\subsection{Choice of Single-Threaded Benchmarks} \label{appendix:onethread}
% ------------------------------------------------

Given the ubiquity of GPUs and multicore CPUs, it may not be obvious why single-threaded experiments are desirable. There are a number of reasons we restrict our focus to CPUs and the single-threaded case:
% We only implemented our algorithm on the CPU for simplicity and
\begin{itemize}
    \item To enable fair comparisons to existing work, particularly the nearest rival, Bolt \cite{bolt}.
    % \item Existing work, particularly our method's nearest rival \cite{bolt}, only uses a single thread.
    \item To facilitate fair comparisons to our work by future authors---single-threaded experiments are much easier to reproduce and extend than multithreaded ones.
    \item Matrix multiplication is embarrassingly parallel with respect to rows of the input and columns of the output. There is therefore nothing ``interesting'' about how our method parallelizes relative to any other; all methods reduce to a single-threaded kernel that can easily be applied to disjoint submatrices. While we certainly could spend the considerable time required to construct and debug multicore benchmarks, this would be unlikely to yield any useful insights.
    \item Parallelization in modern numerical libraries is often managed at a higher level than the lowest-level subroutines. For example, the authors of FBGEMM \cite{fbgemm} state: \textit{``Internally, FBGEMM is intentionally designed not to create any threads. Usually, such a library is intended to be used as a backend by deep learning frameworks, such as PyTorch and Caffe2, that create and manage their own threads.''}\footnote{https://engineering.fb.com/ml-applications/fbgemm/} I.e., a multithreaded library calls into single-threaded subroutines (such as a matrix multiplication function); it is this single-threaded subroutine where we make contributions, and therefore where we focus our experimental efforts. Independent of common practices in modern libraries, this pattern is also the only sensible strategy for small matrices, like many of those we consider---the overhead of launching and joining threads is extremely unlikely to be worth it for sufficiently small matrices. We could perhaps characterize where this breakpoint is, but this is a hardware-specific result that has little to do with our contributions.
    % \item Our method is the only effective AMM approach of which we are aware that requires sublinear memory bandwidth in the matrix size. This is because our encoding function intelligently selects a subset of columns to read. Moreover, our compressed representations are often orders of magnitude smaller than the original matrices. This suggests that, in the memory-bound case, our method would still offer a large advantage over alternatives. However, assessing this in detail requires answering many questions that are extremely workload-specific; e.g., was the matrix A just produced by some subset of cores, and is it therefore already cache-resident? Is it replicated across this cores or sharded? Is it sharded by rows or columns? Can these cores encode their portions of A before sending them to the remaining cores? Or can the overall matrix product be sharded across cores such that no cross-core communication is necessary at all? Must other code be executed that will evict results from caches? Is the encoding of B already in cache for all cores? In other words, assessing the memory-bandwidth bound case requires a detailed specification of the code \textit{around} the matrix product---not merely at the level of the overall task and matrix, but at the level of the exact implementation of the surrounding code. Consequently, we are willing to simply restrict our claims to the compute-bound case (as we do in the paper) rather than attempt to work around all of these lurking variables.
    \item While training of deep neural networks is typically done on GPUs or other accelerators, trained models (including, but not limited to, neural networks) are commonly deployed on smartphones with just CPUs and/or graphics acceleration that is no better than the CPU \cite{fbAtEdge}. Since most of the billions of smartphones in the world tend to be low-end or old, the need to deploy models on CPUs (including those with few cores) is unlikely to change for many years.
    \item Creating, benchmarking, and analyzing a performant implementation of our method for GPUs would require a great deal of engineering work. We plan to create such an implementation in the future, but believe that the many empirical and theoretical results we currently have are more than adequate proof
    of concept and already worth sharing with the community. % Moreover, given both the rapid pace of change in accelerator hardware and diversity of existing hardware, even results on GPUs would quickly
    % \item While it would require a great deal of engineering work to produce a performant GPU version of our method, experiments on CPUs are sufficient proof of concept to suggest
    % \item Our basic ideas can be applied to GPUs or other devices. The only CPU-specific constraint informing our algorithm is the need to have splits at the same level of a tree all use the same index. We of course cannot know precisely how well our ideas work on any given piece of hardware absent direct testing of a high-performance implementation. We therefore plan to create a GPU implementation of our method in future work.
\end{itemize}

% ------------------------------------------------
\subsection{SparsePCA Details}
% ------------------------------------------------

We took steps to ensure that SparsePCA's results were not hampered by insufficient hyperparameter tuning. First, for each matrix product, we tried a range of $\lambda$ values which we found to encompass the full gamut of nearly 0\% to nearly 100\% sparsity: $\lambda \in 2^i, i \in \{-5, -4, -3, -2, -1, 0, 1, 2, 3\}$. Second, because different sparsity patterns may yield different execution times, we report not times from the single matrix SparsePCA produces for a given ($d, \lambda$) pair, but the best times from any of 10 random matrices of the same size and at most the same sparsity. Finally and most importantly, we plot only the Pareto frontier of (speed, quality) pairs produced for a given matrix multiply. I.e., we let SparsePCA cherry-pick its best results on each individual matrix multiply.

% ------------------------------------------------
\subsection{Exact Matrix Multiplication}
% ------------------------------------------------

We also implemented our own matrix product function specialized for tall, skinny matrices. In all cases, we report the timings based on the faster of this function and Eigen's \cite{eigen} matrix multiply function for a given matrix product.

% ------------------------------------------------
\subsection{Additional Baselines}
% ------------------------------------------------

We also tested Frequent Directions / Fast Frequent Directions \cite{liberty_simple_2012, ghashami_frequent_2016, isvd}, many variations of the sampling method of \citet{drineas_fast_2006}, projection using orthogonalized Gaussian random matrices \cite{superbitLSH}, projection using matrices of scaled i.i.d. Rademacher random variables \cite{rademacherJL}, projection using orthonormalized matrices of Rademacher random variables, the co-occurring directions sketch \cite{mroueh_co-occuring_2016}, OSNAP \cite{osnap}, Product Quantization \cite{pq}, and Optimized Product Quantization \cite{opq}.

The poor performance of many of these methods is unsurprising in our setting. Given that we have access to a training set on which to learn the true principal components, the Eckart-Young-Mirsky theorem \cite{eckartYoungMirskyThm} indicates that PCA should outperform any other individual matrix sketching method employing dense projection matrices, at least in the limit of infinite training data. Also, since PQ and OPQ use 256 dense centroids (except in the Bolt / QuickerADC variations), it is also impossible for them to perform well when $\min(D, M)$ is not significantly larger than 256.

% ------------------------------------------------
\subsection{UCR Time Series Archive}
% ------------------------------------------------

We set the number of returned neighbors to 128 (results with 64 and 256 were similar). We omitted datasets with fewer than 128 training examples, since it is not possible for Stochastic Neighbor Compression to draw 128 samples without replacement in this case.

In addition to being a large, public corpus of over a hundred datasets from a huge variety of different domains, the UCR Time Series Archive also has the advantage that it can be used to produce matrix multiplication tasks of a fixed size. This is necessary for meaningful comparison of speed versus accuracy tradeoffs across datasets. We constructed training and test matrices $\tilde{\A}$ and $\A$ by resampling each time series in each dataset's train and test set to a length of $320$ (the closest multiple of 32 to the median length of 310). We obtained the matrix $\B$ for each dataset by running Stochastic Neighbor Compression \cite{snc} on the training set with an RBF kernel of bandwidth one. We set the number of returned neighbors to 128 (results with 64 and 256 were similar), yielding a $\B$ matrix of size $320 \times 128$. %See Appendix~\ref{sec:experimentDetails} for additional details.
Since different datasets have different test set sizes, all results are for a standardized test set size of 1000 rows. We wanted the length to be a multiple of 32 since existing methods operate best with sizes that are either powers of two or, failing that, multiples of large powers of two.

We approximate Euclidean distances using the identity $\norm{\vec{x} - \vec{y}}_2^2 = \norm{\vec{x}}_2^2 - 2\vec{x}^\top \vec{y} + \norm{\vec{y}}_2^2$. We approximate only the inner products $\vec{x}^\top \vec{y}$, since $\norm{\vec{y}}_2^2$ can be precomputed for fixed exemplars $\vec{y}$ and $\norm{\vec{x}}_2^2$ doesn't affect the class prediction since it is constant across all exemplars for a given input $\vec{x}$.

% since having a length that is a multiple of at least 8 is the best-case scenario for most existing methods, and 320 is a ``rounder'' number than 312.

% ------------------------------------------------
\subsection{Caltech101}
% ------------------------------------------------

We only extracted valid windows---i.e., never past the edge of an image. We extracted the windows in CHW order, meaning that scalars from the same color channel were placed at contiguous indices. The ``first'' images are based on filename in lexicographic order.

We used pairs of filters because using a single filter would mean timing a matrix-vector product instead of a matrix-matrix product.

To allow meaningful speed comparisons across images, we resized and center-cropped each image to $224 \times 224$ as commonly done in image classification pipelines \cite{resNet,resnet2,densenet}. We then extracted sliding windows of the appropriate size and used each (flattened) window as one row of $\tilde{\A}$ or $\A$. We similarly flattened the filters, with each set of coefficients forming one column of $\B$. In both cases, $\B$ has two columns---this is because using a single filter would mean timing a matrix-vector product instead of a matrix-matrix product. Two columns also made sense since Sobel filters are often used in horizontal and vertical pairings, and Gaussian filters are often used together to perform difference-of-Gaussians transforms.

Even though the RGB values at each position are naturally unsigned 8-bit integers, we allowed all rival methods to operate on them as 32-bit floating point, without including the conversion when timing them. Because it only requires checking whether values are above a threshold, \oursp can operate on 8-bit data directly.

% ------------------------------------------------
\subsection{Why Not Speed Up Whole Neural Nets?}
% ------------------------------------------------

Using our ideas to accelerate overall neural networks and other layer types would be a valuable contribution. In fact, we are actively working on this problem. However, as we state in the introduction and problem statement, our focus in this paper is \textit{approximate matrix multiplication} (AMM) and we deliberately make no claim about accelerating entire neural networks or convolutional layers. We limit our scope in this way for several reasons:
\begin{enumerate}
\item Approximate matrix multiplication is an established research problem of general interest independent of deep learning.
\item Lifting a method from accelerating a single layer to an overall network is challenging. Just as scalar quantization of network parameters is simple for a single layer but an active area of research for an entire network, so too is using our method on multiple layers at once an open research problem. For example, it is not clear how to deal with the fact that distributions of activations change throughout training, or how to efficiently incorporate our non-differentiable hash function.
We could show how to accelerate one internal FC layer in a network, but we don’t want to risk misleading the reader---it would be unclear what conclusions to draw from such results, particularly given the difficulty of retraining / fine-tuning without introducing many lurking variables (c.f., \cite{blalock2020}).
\item It is correct that convolution can be reduced to GEMM using im2col, and that accelerating convolution using our ideas would be a valuable contribution. However, state-of-the-art algorithms for convolution exploit structure that is not available to general matrix multiply algorithms. To match the performance of specialized Winograd, direct, FFT-based, and hybrid convolution schemes that do exploit this additional structure, we would have to make modifications to our approach that would make it less general. For example, the individual spatial positions should be encoded only once, and then reused at multiple filter positions. Regarding Section 5.5: while we do test our method on matrices of flattened image patches, we do not claim that the overall pipeline of flattening + matrix multiply constitutes a state-of-the-art convolution method---we only claim that using our method in this pipeline outperforms using other AMM methods there.
% \item The paper is already heavily space constrained. Even if we had already completed all the additional research and engineering work necessary to generalize our approach to convolutional layers and full neural networks, there simply wouldn't be space to explain the method or include the relevant results.
\end{enumerate}

In short, while we believe that our ideas show great promise for accelerating full neural networks and more layer types, making this happen requires much more research.% that cannot cleanly be folded into the current paper.
% Furthermore, because the problem we consider is important and established on its own, we do not see it as necessary to expand the scope of the present work.

% ------------------------------------------------
\subsection{Additional Results}
% ------------------------------------------------

In Section~\ref{sec:results}, we showed the classification accuracy as a function of wall time for the CIFAR-10 and CIFAR-100 softmax classifiers, as well as on the UCR datasets. In Figure~\ref{fig:cifarNMSE} and Figure~\ref{fig:ucrNMSE}, we instead show normalized mean squared error versus time. In Figure~\ref{fig:cifarOps} and Figure~\ref{fig:caltechOps}, we show accuracy or NMSE versus number of operations performed, where one operation is either one multiply-add or one table lookup, depending on the method. The first two figures illustrate that NMSE is closely related to classification accuracy, but with imperfect NMSE still yielding excellent accuracy in many cases. The second two figures show that our method's superior results are not merely caused by the use of faster CPU instructions, but also by the use of fewer basic operations at the algorithm level.

\begin{figure}[h]
\begin{center}
\includegraphics[width=\linewidth]{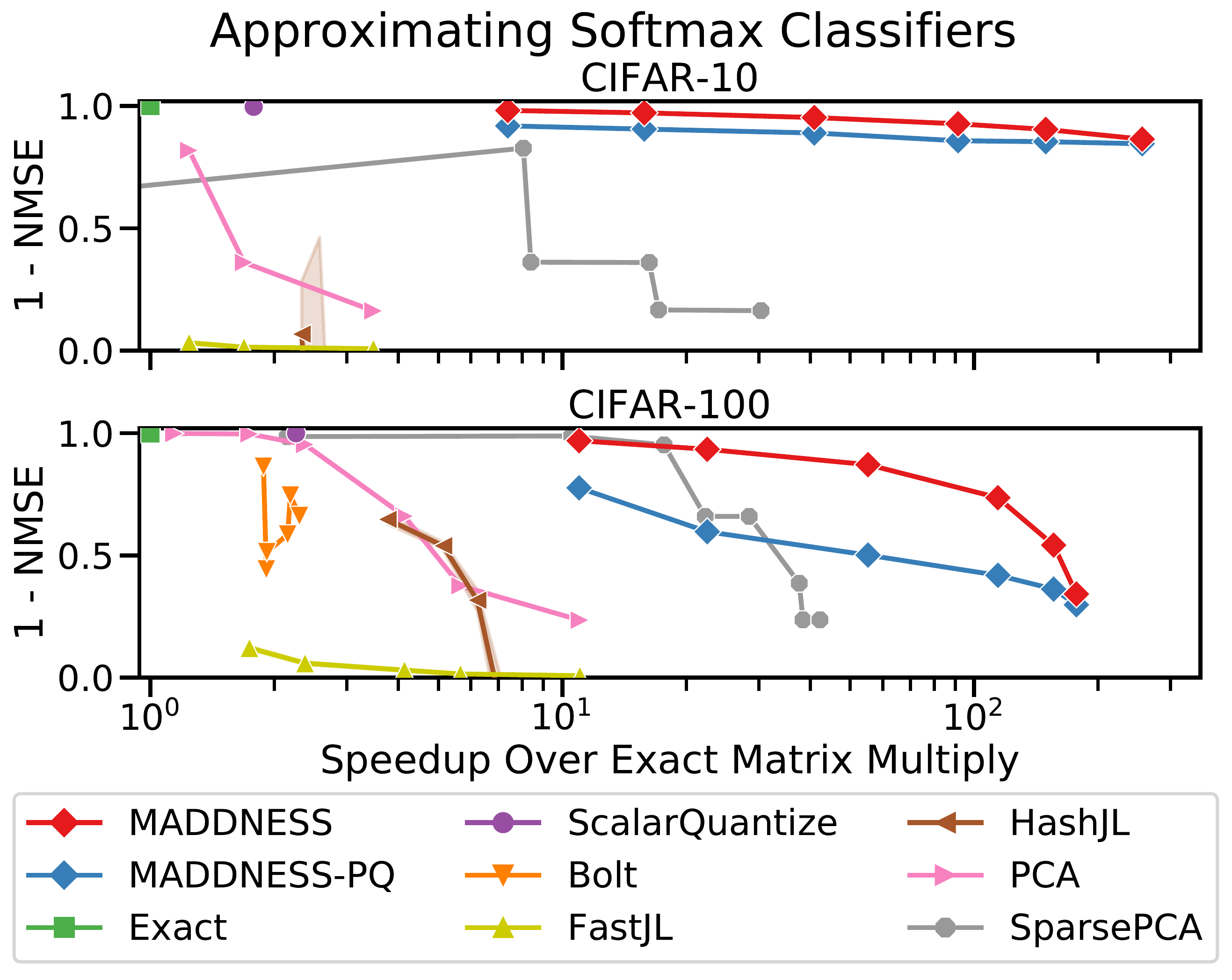}
\caption{\oursp achieves a far better speed versus squared error tradeoff than any existing method when approximating two softmax classifiers. These results parallel the speed versus classification accuracy results, except that the addition of our ridge regression is much more beneficial on CIFAR-100.}
\label{fig:cifarNMSE}
\end{center}
\end{figure}

\begin{figure}[h]
\begin{center}
\includegraphics[width=\linewidth]{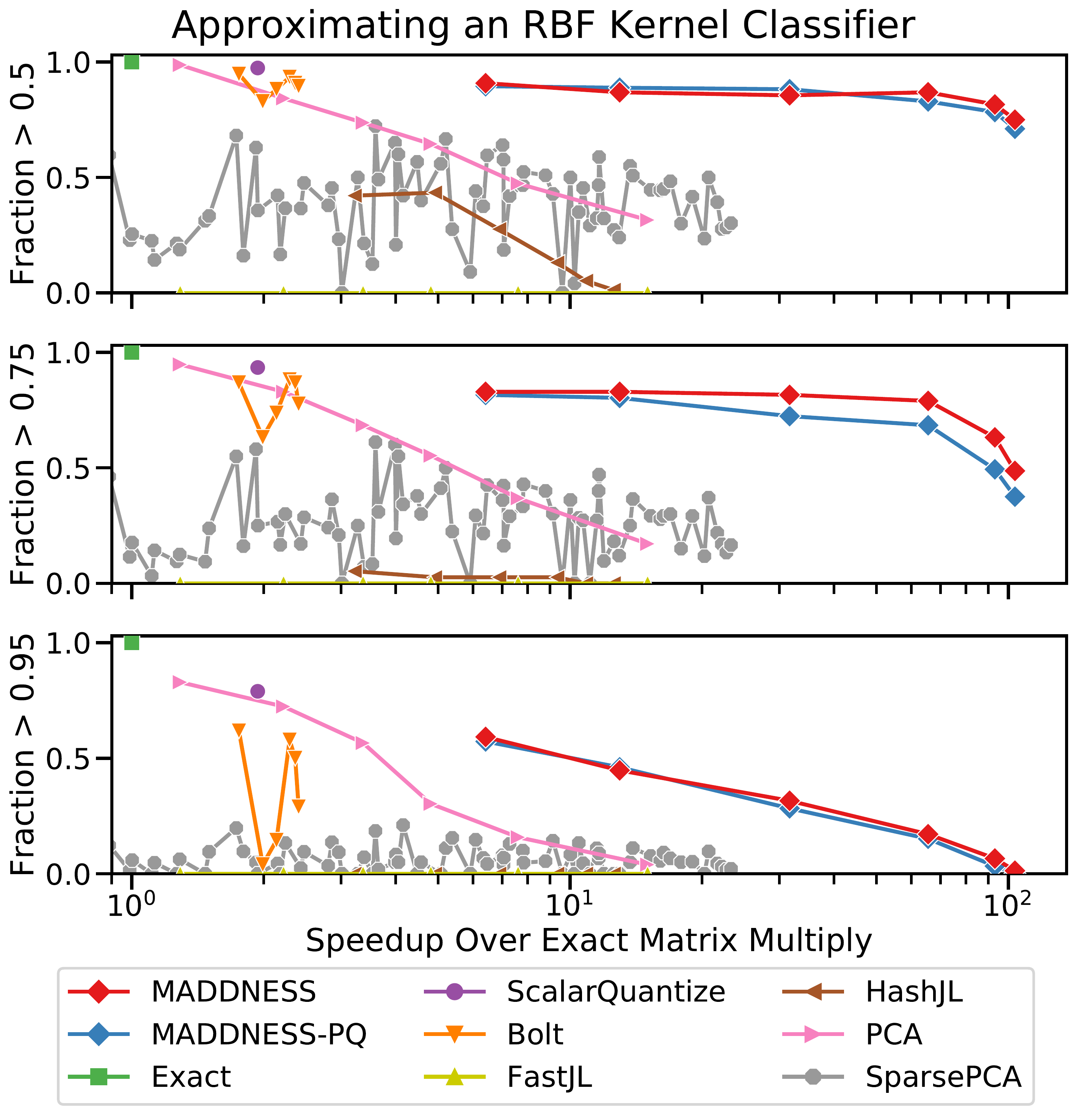}
\caption{\oursp achieves the lowest squared error at high speedups on the UCR datasets. These results parallel the speed versus classification accuracy results.}
\label{fig:ucrNMSE}
\end{center}
\end{figure}

\begin{figure}[h]
\begin{center}
\includegraphics[width=\linewidth]{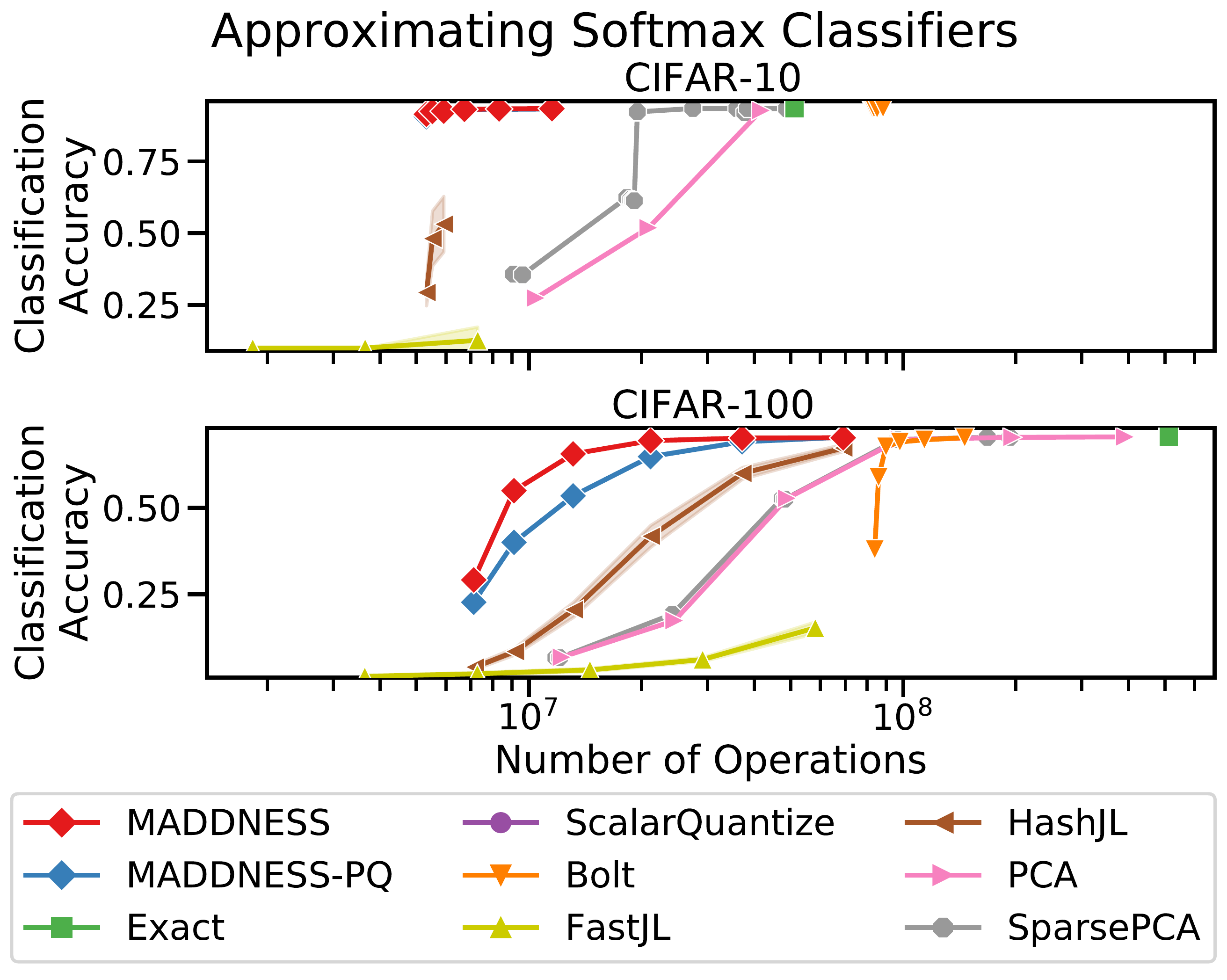}
\caption{\oursp achieves the best speed versus accuracy tradeoff on the CIFAR datasets of any method even when speed is measured as number of operations instead of wall time. Note that fewer operations with a high accuracy (up and to the left) is better.}
\label{fig:cifarOps}
\end{center}
\end{figure}

\begin{figure}[h]
\begin{center}
\includegraphics[width=\linewidth]{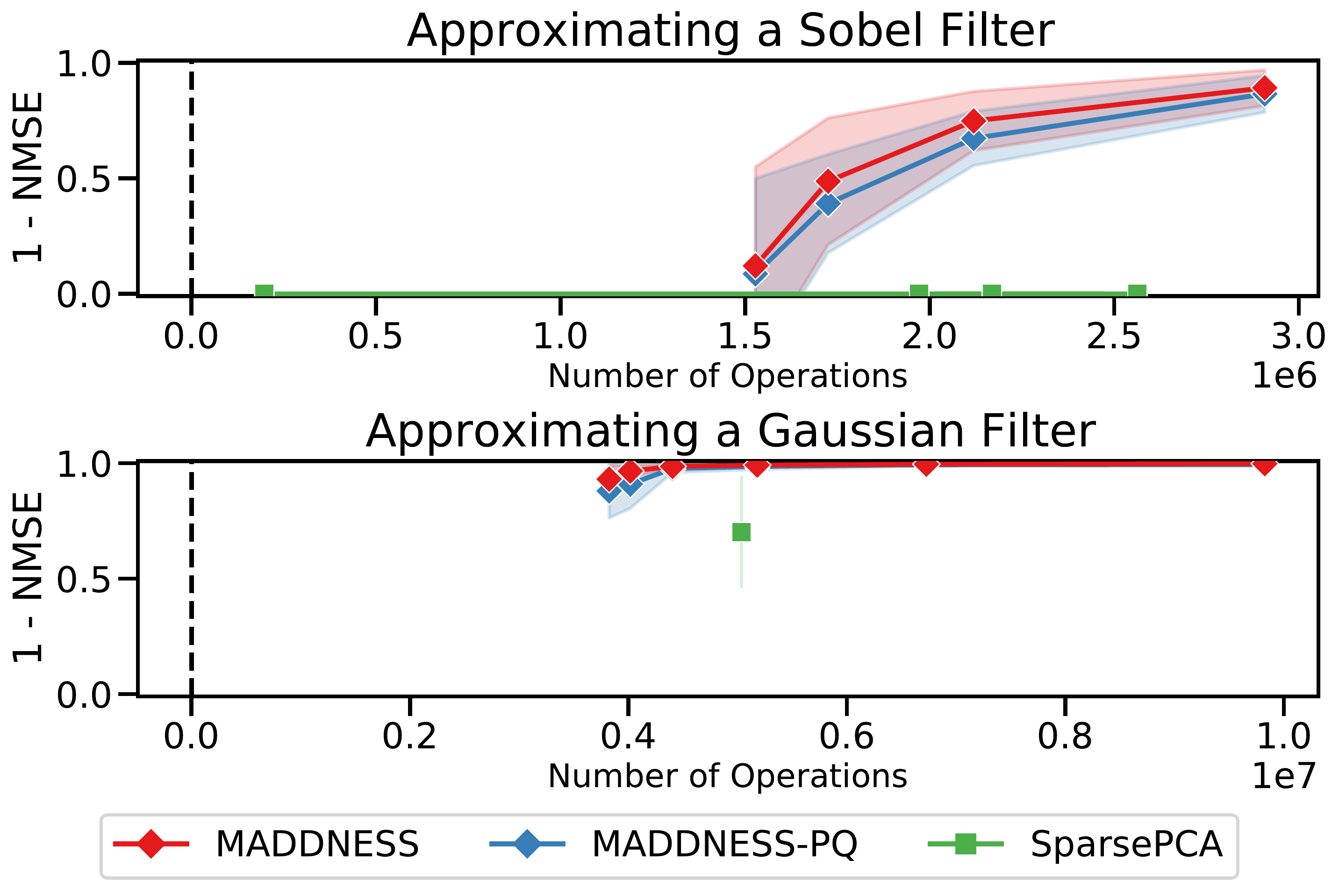}
\caption{\oursp still achieves the best speed versus squared error tradeoff on the image processing tasks when speed is measured as number of operations instead of wall time.}
\label{fig:caltechOps}
\end{center}
\end{figure}

% ================================================================
% \vfill\break  % columnbreak
% \vfill
\clearpage
\section{Theoretical Analysis of \ours} \label{sec:maddnessMath}
% ================================================================

% ------------------------------------------------
\subsection{Complexity}
% ------------------------------------------------

% Our encoding function $g(\A), \A \in \R^{N \times D}$ has complexity $\Theta(NC)$, since it does a constant amount of work per row per codebook. Our table creation function $h(\B), \B \in \R^{D \times M}$ has complexity $\Theta(MKCD)$, since it must compute the inner product between each column of $\B$ and $KC$ prototypes of length $D$. This is a factor of $C$ worse than PQ since we do not require the prototypes for different codebooks to have disjoint nonzero indices. However, as discussed in section~\ref{sec:problemStatement}, this reduction in the speed of $h(\cdot)$ is not a concern. Finally, the complexity of our aggregation function $f(\cdot)$ is $\Theta(NCM)$, since it performs $C$ table lookups for each of $M$ output columns and $N$ output rows. This means our overall algorithm has complexity $\Theta(MC(KD + N))$, which reduces to $\Theta(NCM)$ since we fix $K = 16$, and our problem statement requires $N \gg D$.

Our encoding function $g(\A), \A \in \R^{N \times D}$ has complexity $\Theta(NC)$, since it does a constant amount of work per row per codebook. Our table creation function $h(\B), \B \in \R^{D \times M}$ has complexity $\Theta(MKCD)$, since it must compute the inner product between each column of $\B$ and $KC$ prototypes of length $D$. This is a factor of $C$ worse than PQ since we do not require the prototypes for different codebooks to have disjoint nonzero indices. However, this reduction in the speed of $h(\cdot)$ is not a concern because $N \gg M, D$; moreover, the matrix $\B$ is often known ahead of time in realistic settings, allowing $h(\B)$ to be computed offline. Finally, the complexity of our aggregation function $f(\cdot)$ is $\Theta(NCM)$, since it performs $C$ table lookups for each of $M$ output columns and $N$ output rows. This means our overall algorithm has complexity $\Theta(MC(KD + N))$, which reduces to $\Theta(NCM)$ since we fix $K = 16$ and our problem statement requires $N \gg D$.

% ------------------------------------------------
\subsection{Proof of Generalization Guarantee}
% ------------------------------------------------

In this section, we prove Theorem~\ref{thm:maddness}, restated below for convenience.

\begin{theorem*}[Generalization Error of \ours] \label{thm:maddnessAppendix}
Let $\Dcal$ be a probability distribution over $\R^D$ and suppose that \oursp is trained on a matrix $\tilde{\A} \in R^{N \times D}$ whose rows are drawn independently from $\Dcal$ and with maximum singular value bounded by $\sigma_A$. Let $C$ be the number of codebooks used by \oursp and $\lambda > 0$ be the regularization parameter used in the ridge regression step. %Then for any vector $\b$, any vector $\a \sim \Dcal$, and any $0 < \delta < 1$, we have with probability at least $1 - \delta$ that
Then for any $\b \in \R^D$, any $\a \sim \Dcal$, and any $0 < \delta < 1$, we have with probability at least $1 - \delta$ that
 % Then for all vectors $\b$ and any vector $\a \sim \Dcal$, we have with probability at least $1 - \delta$, $0 < \delta < 1$,
\begin{align*}
    \begin{split}
    \E_{\Dcal}[&\Lcal(\a, \b)] \le \E_{\tilde{\A}}[\Lcal(\a, \b)] + \\
    &\frac{C \sigma_A \norm{\b}_2}{2 \sqrt{\lambda}} \left(
        \frac{1}{256} +
        \frac{
            8 +
            \sqrt{
                % C (4\ceil{\log_2(D)} + 256) \log{2} -\log{\delta}
                \nu(C, D, \delta)
            }
        }{\sqrt{2n}}
    \right)
    \end{split}
% \end{equation}
\end{align*}
% where
% where $\Lcal(\a, \b) \triangleq |\a^\top \b - \alpha f(g(\a), h(\b)) - \mat{\beta}|$ (c.f., Equation~\ref{eq:objective}).
where $\Lcal(\a, \b) \triangleq |\a^\top \b - \alpha f(g(\a), h(\b)) - \mat{\beta}|$, $\alpha$ is the scale used for quantizing the lookup tables, $\mat{\beta}$ is the constants used in quantizing the lookup tables plus the debiasing constant of Section~\ref{sec:aggregate}, and
\begin{align*}
    \nu(C, D, \delta) \triangleq C (4\ceil{\log_2(D)} + 256) \log{2} -\log{\delta}.
\end{align*}
% \begin{align}
%     \begin{split}
%     \E_{\Dcal}[ \Lcal(\a, \b)] \le \E_{\tilde{\A}}[\Lcal(\a, \b)] +
%     \frac{C \sigma_A \norm{\b}_2}{2 \sqrt{\lambda}} \bigg(
%         \frac{1}{256} + \\
%         \frac{
%             8 +
%             \sqrt{
%                 % C ((4 + \frac{2}{C})\ceil{\log_2(D)} + 284) \log{2} -\log{\delta}
%                 C (4\ceil{\log_2(D)} + 256) \log{2} -\log{\delta}
%             }
%         }{\sqrt{2n}}
%     % \right)
%     \bigg)
%     \end{split}
% \end{align}
% where $\Lcal(\a, \b) \triangleq |\a^\top \b - \alpha f(g(\a), h(\b)) - \mat{\beta}|$, $\alpha$ is the scale used for quantizing the lookup tables, and $\mat{\beta}$ is the constants used in quantizing the lookup tables plus the debiasing constant of Section~\ref{sec:aggregate}.

\end{theorem*}

The proof relies on the observation that \ours's training procedure can be decomposed into two sequential subroutines: \texttt{Maddness-Build-Tree}, which learns the function $g(\a)$ by constructing a binary decision tree, and \texttt{Maddness-Regress}, which learns the function $h(\b)$ by optimizing a prototype matrix $\P$ such that $g(\tilde{\A}) \P \approx \tilde{\A}$.
% regressing the output $\mat{G} \triangleq g(\tilde{\A})$ onto the.
%prototype matrix $\P$ used by using ridge regression conditioned on the output of \texttt{Maddness-Build-Tree} on the training matrix $\tilde{\A}$.
% We will prove an overall generalization guarantee for \oursp by first providing a guarantee for \texttt{Maddness-Regress} for a fixed \texttt{Maddness-Build-Tree} hypothesis, and then union bounding over the hypothesis space for \texttt{Maddness-Build-Tree}.
This observation allows us to prove~\ref{thm:maddness} by first providing a guarantee for \\ \texttt{Maddness-Regress} for a fixed \texttt{Maddness-Build-Tree} hypothesis, and then union bounding over the hypothesis space for \texttt{Maddness-Build-Tree}. Bounding the size of the hypothesis space is straightforward (Lemma~\ref{lemma:ntrees}), so the bulk of this section focuses on providing a guarantee for \texttt{Maddness-Regress}. We must also prove a bound on the loss contributed by quantizing the lookup tables array $\P^\top \b$. % There is also a need to bound the errors introduced by quantizing the lookup tables $\P^\top \b$, but this error is small and.

\begin{lemma}[Number of Hypotheses for \texttt{Maddness-Build-Tree}] \label{lemma:ntrees}
Let $C$ be the number of codebooks used by \oursp and let $D$ be the number of columns in the matrix $\tilde{\A}$ on which \oursp is trained. Then there are at most $2^{C (4\ceil{\log_2(D)} + 256)}$ unique trees that \texttt{Maddness-Build-Tree} can generate.
\end{lemma}
\begin{proof}
\texttt{Maddness-Build-Tree} learns four sets of parameters for each of the $C$ trees it produces: split indices, split offsets, split scales, and split values.

There are four split indices per tree because there is one for each of the tree's four levels. Each index requires $\ceil{\log_2(D)}$ bits to store, so the split indices require a total of $4 \ceil{\log_2(D)}$ bits per tree. For each split index, there is one split offset and scale, used to map floating point data in an arbitrary range to the interval $[0, 255]$ to match up with the 8-bit split values.

The offsets require at most 25 bits each for 32-bit floating point data, since the low seven bits can be dropped without affecting the post-scaling quantized output. The scales are constrained to be powers of two, and so require at most nine bits for non-subnormal 32-bit floating point inputs (which have one sign bit and eight exponent bits). The offsets and scales together therefore contribute $4 (25 + 9) = 136$ bits per tree.

There are 15 split values because there is one for the root of each tree, then two for the second level, four for the third, and eight for the fourth. Each split value is stored using eight bits, so each tree requires $15 \cdot 8 = 120$ bits for split values. The total number of bits used for all trees is therefore $C(4 \ceil{\log_2(D)} + 256)$. Note that the constant 256 being a power of two is just an accident of floating point formats. The claimed hypothesis count follows from the number of expressible hypotheses being at most two to the power of the largest number of bits used to store any hypothesis.
% Finally, it is necessary to encode the number of bits used to store split indices, $\ceil{\log_2(D)}$; assuming $D$ can be stored in 32 bits, this is at most 5 bits.
\end{proof}

% \begin{theorem}[[Bartlett and Mendelson 2002]] let $\Fcal$ be a class of functions, let $\Scal$ be a set of $n$ samples drawn i.i.d. from some distribution $\Dcal$, and let $\Lcal(f), f \in \Fcal$ be a loss function with Lipschitz constant $l$. Then for any $\delta, 0 < \delta < 1$, it holds for all $f \in \Fcal$ with probability at least $1 - \delta$ that
% \begin{align}
%     \E_\Dcal[\Lcal{f}] \le \E_S[\Lcal(f)] + 2 l \Rcal_n(\Fcal) + l \sqrt{\frac{\log(1 / \delta)}{2n}}
% \end{align}
% \end{theorem}

We now turn our attention to bounding the errors of the regression component of training. Our strategy for doing so is to bound the largest singular value of the learned matrix of prototypes $\P$. Given such a bound, the norms of both $g(\a)^\top \P$ and $\P^\top \b$ can be bounded.

% and $\Y \in \R^{D \times M}
\begin{lemma}[Regularized Pseudoinverse Operator Norm Bound] \label{lemma:pinvBound}
Let $\X \in \R^{N \times D}$
% $\lambda$
% $\mat{I}$
% $\W$
be an arbitrary matrix with finite elements. Then every singular value $\sigma_i$ of the matrix $\Z \triangleq (\X^\top \X + \lambda \mat{I})^{-1} \X^\top$, $\lambda > 0$ is at most $\frac{1}{2\sqrt{\lambda}}$.
\end{lemma}
\begin{proof} Let $\U \Sigm \Vt$ be the singular value decomposition of $\X$. Then we have
\begin{align}
    \Z &= (\X^\top \X + \lambda \I)^{-1} \X^\top   \\
    &= (\V \Sigm \Ut \U \Sigm \Vt + \lambda \I)^{-1} \V \Sigm \Ut \\
    &= (\V \Sigm^2 \Vt + \lambda \I)^{-1} \V \Sigm \Ut \\
    &= (\V \Sigm^2 \Vt + \V \lambda \I \Vt)^{-1} \V \Sigm \Ut \label{step:wrapI} \\
    &= (\V \Sigm_\lambda \Vt)^{-1} \V \Sigm \Ut \\
    &= \V \Sigm_\lambda^{-1} \Vt \V \Sigm \Ut \\
    &= \V \Sigm_\lambda^{-1} \Sigm \Ut \\
    &= \V \Sigm^\prime \Ut
\end{align}
where $\Sigm_\lambda \triangleq \Sigm^2 + \lambda \I$ and $\Sigm^\prime \triangleq (\Sigm^2 + \lambda \I)^{-1} \Sig$. Step \ref{step:wrapI} follows from the equality $\V \lambda \I \Vt = \lambda \V \Vt = \lambda \I$. Because the matrices $\V$ and $\Ut$ are orthonormal and $\Sigm^\prime$ is diagonal, the singular values of $\Z$ are equal to the diagonal entries of $\Sigm^\prime$. Each entry $\sigma_i^\prime$ is equal to
\begin{align}
    \sigma_i^\prime = \frac{\sigma_i}{\sigma_i^2 + \lambda}.
\end{align}
This expression attains its maximal value of $\frac{1}{2\sqrt{\lambda}}$ when $\sigma_i^2 = \lambda$.
\end{proof}

\begin{lemma}[Ridge Regression Singular Value Bound] \label{lemma:ridgeBound}
Let $\X \in \R^{N \times D}$ and $\Y \in \R^{D \times M}$ be arbitrary matrices and let $\W \triangleq (\X^\top \X + \lambda \mat{I})^{-1} \X^\top \Y$, $\lambda > 0$ be the ridge regression weight matrix. Then $\norm{\W}_\inf \le \frac{\norm{\Y}_\inf}{2 \sqrt{\lambda}}$, where $\norm{\cdot}_\inf$ denotes the largest singular value.
\end{lemma}
\begin{proof}
Observe that $\W = \Z \Y$, where $\Z \triangleq (\X^\top \X + \lambda \mat{I})^{-1} \X^\top$. Then by applying Lemma~\ref{lemma:pinvBound} and recalling that Schatten norms are submultiplicative, we have
\begin{align}
    \norm{\W}_\inf \le \norm{\Z}_\inf \norm{\Y}_\inf \le \frac{\norm{\Y}_\inf}{2\sqrt{\lambda}}.
\end{align}
\end{proof}

\begin{lemma}[Bound on \oursp Embedding Norm] \label{lemma:embedNorm}
Let $\g = g(\a)$ be the encoding of an arbitrary vector $\a$ using $C$ codebooks and let $\P$ be the prototype matrix learned by \oursp using training matrix $\tilde{\A}$ with ridge regression parameter $\lambda > 0$. Then
\begin{align}
    \norm{\g^\top \P}_2 \le \frac{C}{2 \sqrt{\lambda}} \norm{\tilde{\A}}_{\inf}
\end{align}
where $\norm{\tilde{\A}}_{\inf}$ denotes the largest singular value of $\tilde{\A}$.
\end{lemma}
\begin{proof} We have
\begin{align}
    \norm{\g^\top \P}_2 &\le \norm{\g}_2 \norm{\P}_\inf \\
    &= C \norm{\P}_\inf \\
    &\le \frac{C}{2 \sqrt{\lambda}} \norm{\tilde{\A}}_{\inf}.
\end{align}
The first step follows from Cauchy-Schwarz. The second follows from $\g$ being zero except for exactly $C$ ones. The last is an application of Lemma~\ref{lemma:ridgeBound}.
\end{proof}

\begin{lemma}[Maximum Table Quantization Loss] \label{lemma:lutQuantBound}
Let $\ahat = g(\a)^\top\P$, where $g(\cdot)$ and $P$ are trained using $C$ codebooks and ridge regression penalty $\lambda > 0$ on a matrix $\tilde{\A}$ with maximum singular value at most $\sigma_A$, and $\a \in \R^D$ is an arbitrary vector. Then for any vector $\b \in \R^D$, $|\ahat^\top \b - \yhat| < \frac{C \sigma_A \norm{\b}_2}{512 \sqrt{\lambda}} $, where %$\yhat \triangleq g(\a)^\top (\P^\top \b)$.
$\yhat \triangleq \alpha g(\a)^\top g(\b) + \mat{\beta}$ is \ours's approximation to $\a^\top \b$. $\alpha$ and $\mat{\beta}$ are the scale and offsets used to quantize the lookup tables.
%defined in Equation~\ref{eq:objective} except without the averaging debias constant added to $\mat{\beta}$.
\end{lemma}
\begin{proof} If \oursp had infinite-precision lookup tables, $\yhat$ would exactly equal $\ahat^\top \b$. We therefore need only bound the error introduced by the quantization. By Lemma~\ref{lemma:embedNorm}, $\norm{\ahat}_2 \le \frac{C \sigma_A}{2 \sqrt{\lambda}} $. This implies that
\begin{align}
    \norm{\ahat^\top \b} \le \frac{C \sigma_A \norm{\b}_2}{2 \sqrt{\lambda}}
\end{align}
and therefore
\begin{align}
    \frac{-C \sigma_A \norm{\b}_2}{2 \sqrt{\lambda}} \le \ahat^\top \b \le \frac{C \sigma_A \norm{\b}_2}{2 \sqrt{\lambda}}.
\end{align}
For each of the $C$ codebooks, this means that the value to be quantized lies in the interval $[\frac{-\sigma_A \norm{\b}_2}{2 \sqrt{\lambda}}, \frac{\sigma_A \norm{\b}_2}{2 \sqrt{\lambda}}]$ of width $\frac{\sigma_A \norm{\b}_2}{\sqrt{\lambda}}$.
Because \oursp quantizes the lookup tables such that largest and smallest entries for any row of $\P$ are linearly mapped to $255.5$ and $-0.5$,\footnote{We use $255.5$ and $-0.5$ rather than $255$ and $0$ because the latter only guarantees that a point is within $1 / 510$ of the interval width, not $1 / 512$. This is not an important choice and either option would be fine.} respectively, the worst-case quantization error is when the quantized value lies exactly between two quantization levels. % (as opposed to lying outside the linearly rescaled interval).
We therefore need to compute the largest possible gap between a value and its quantization. Using 256 quantization levels, the largest possible gap is $1 / (256 / .5) = 1/512$ of the interval width. Multiplying by the above interval width yields a maximum quantization error for a given codebook of $\frac{\sigma_A \norm{\b}_2}{512 \sqrt{\lambda}}$.
 % Using the worst-case interval above and observing that 256 quantization levels imply a maximum quantization error of $1 / (256 / .5) = 1/512$ of the interval width gives us a maximum error for each codebook of $\frac{\sigma_A \norm{\b}_2}{512 \sqrt{\lambda}}$.
Because the errors in each subspace may not agree in sign, their sum is an upper bound on the overall quantization error.
\end{proof}

At this point, we have all of the pieces necessary to prove a generalization guarantee for \texttt{Maddness-Regress} save one: a theorem linking the norms of the various vectors and matrices involved to a probabilistic guarantee. \citet{kakadeLinear} provide such a gaurantee, based on Rademacher complexity \cite{rademacherOrig}.

\begin{theorem}[\cite{kakadeLinear}, Corollary 5] \label{thm:linearGeneralize}
Let $\Fcal = \{\w^\top \x : \norm{\w}_2 \le W \}$ be the class of linear functions with bounded $L_2$ norms, let $\Scal$ be a set of $n$ samples drawn i.i.d. from some distribution $\Dcal$ over the $L_2$ ball of radius $X$,
% such that $\x \sim \Dcal \implies \norm{\x}_2 \le X$ almost surely for some $X$
and let $\Lcal(f), f \in \Fcal$ be a loss function with Lipschitz constant $L$. Then for any $0 < \delta < 1$, it holds with probability at least $1 - \delta$  over the sample $\Scal$ that
\begin{align}
    \E_\Dcal[\Lcal(f)] \le \E_S[\Lcal(f)] + \frac{L X W}{\sqrt{2n}} \left( 8 + \sqrt{-log(\delta)} \right) .
    % l \Rcal_n(\Fcal) + l \sqrt{\frac{\log(1 / \delta)}{2n}}
\end{align}
\end{theorem}

We can now obtain our desired guarantee for the regression step.

% \begin{lemma}[Rademacher Complexity and Maximum Loss of \texttt{Maddness-Regress}].
\begin{lemma}[Generalization Error of \texttt{Maddness-Regress}] \label{lemma:regGeneralize}
% $\Dcal: \R^{N \times D} \rightarrow \R_+$
% \tilde{\A} \sim \Dcal,

% Let $\tilde{\A} \in \R^{N \times D}$ be the matrix on which \oursp is trained. Let $\G = g(\tilde{\A}) \in \R^{N \times 16C}$ be the encoding of $\tilde{\A}$ using a fixed (data-independent) set of $C$ trees output by \\
% \texttt{Maddness-Build-Tree}; let $\sigma_A$ be an upper bound on the largest singular value of any matrix $\A$ drawn from $\Dcal$ with $N$ rows; let $\lambda > 0$ be the regularization parameter of the ridge regression used by \texttt{Maddness-Regress}; and let $\Lcal(\a, \b) \triangle |\a^\top \b - f(g(\a), h(\b))|$. Then for any row $\a$ of any matrix $\A \sim \Dcal$, and any vector $\b, \norm{\b}_2 \le L_{\b}$, it holds with probability at least $1 - \delta$, $0 < \delta < 1$ over the training matrix $\tilde{\A}$ that:
Let $\Dcal$ be a probability distribution over $\R^D$ and suppose that \oursp is trained on a matrix $\tilde{\A} \in R^{N \times D}$ whose rows are drawn independently from $\Dcal$ and with maximum singular value bounded by $\sigma_A$. Let $C$ be the number of codebooks used by \oursp and $\lambda > 0$ the regularization parameter used in the ridge regression step. Further let $g(\a)$ be a fixed (data-independent) function and $\Lcal(\a, \b) \triangleq |\a^\top \b - f(g(\a), h(\b))|$. Then for all vectors $\b$, any vector $\a \sim \Dcal$, and any $0 < \delta < 1$, we have with probability at least $1 - \delta$ that
\begin{align} \label{eq:regressGeneralize}
    \begin{split}
    \E_{\Dcal}[ & \Lcal(\a, \b)] \le \E_{\tilde{\A}}[\Lcal(\a, \b)] + \frac{C \sigma_A \norm{\b}_2}{512 \sqrt{\lambda}} + \\
    &\frac{C \sigma_A \norm{\b}_2}{2\sqrt{2n \lambda}} \left( 8 + \sqrt{-log(\delta)} \right).
    \end{split}
\end{align}
\end{lemma}

\begin{proof}
The output of \texttt{Maddness-Regress} can be decomposed into
\begin{align}
    \yhat \triangleq f(g(\a), h(\b)) = \g^\top \P \b + \eps + \zeta
\end{align}
where $\g = g(\a)$, $\P$ is the matrix of prototypes, $\eps$ is data-independent noise from the averaging process\footnote{We continue to make the assumption that the least significant bits of the lookup table entries are independent Bernoulli(0.5) random variables, which is nearly true in practice. Even if this assumption does not hold, this noise does not contribute to the generalization gap unless it differs between train and test sets.}, and $\zeta$ is noise from quantizing the lookup table entries. By Lemma~\ref{lemma:lutQuantBound}, $\zeta \le \frac{C \sigma_A \norm{\b}_2}{512 \sqrt{\lambda}}$ (accounting for the second term in Equation~\ref{eq:regressGeneralize}). We therefore need only obtain a guarantee for $|\g^\top \P \b - \a^\top \b|$. Defining $\w \triangleq \P \b$, we see that \texttt{Maddness-Regress} is a linear model, and therefore subject to Theorem~\ref{thm:linearGeneralize}. Given an upper bound on the Lipschitz constant of the loss, a bound on the $L_2$ norm of $\g$, and a bound on the $L_2$ norm of $\w$, we can apply this theorem. The Lipschitz constant for the absolute loss is 1. The $L_2$ norm of $\g$ is exactly $C$. The $L_2$ norm of $\w$ can be bounded as
\begin{align}
    \norm{\w}_2 &= \norm{\P \b}_2 \le \norm{\P}_\inf \norm{\b}_2
    \le \frac{\sigma_A \norm{\b}_2 }{2\sqrt{\lambda}}
\end{align}
using Lemma~\ref{lemma:ridgeBound}.
\end{proof}

Using this lemma, the proof of Theorem~\ref{thm:maddness} is immediate; we begin with Lemma~\ref{lemma:regGeneralize} and simply union bound over all $2^{C (4\ceil{\log_2(D)} + 120)}$ hypotheses from Lemma~\ref{lemma:ntrees}.

\end{document}